\documentclass{article} 
\usepackage{arxiv} 

\RequirePackage{fix-cm}

\usepackage{graphicx}
\usepackage{balance} 
\usepackage{tikz}
\usepackage{pgf}
\usepackage{float}
\usepackage[utf8]{inputenc}
\usepackage{pgfplots}
\usepackage{pgfplots}
\usepgfplotslibrary{patchplots,colormaps}
\pgfplotsset{compat=1.9} 
\DeclareUnicodeCharacter{2212}{−}
\usepgfplotslibrary{groupplots,dateplot}
\usetikzlibrary{patterns,shapes.arrows}
\usepackage{graphicx}
\usepackage{caption}
\usepackage{subcaption}
\usepackage{array}
\usepackage{amsmath}
\usepackage{amssymb}
\usepackage{amsfonts}
\usepackage{bm}
\usepackage{xcolor}
\usepackage{multirow,array}
\usepackage{comment}
\usepackage{url} 
\usepackage{tikz}
\usepackage{pgfplots}
\usepackage[normalem]{ulem}
\usetikzlibrary{matrix,positioning}
\usepackage{makecell}
\usepackage[flushleft]{threeparttable}
\usepackage{multicol}
\usepackage{graphicx}
\usepackage[affil-it]{authblk}

\usepackage{graphicx}
\usepackage{xcolor}
\usepackage{caption}
\usepackage{subcaption}
\usepackage{amsmath}
\usepackage{amsfonts}
\usepackage[section]{placeins}
\usepackage{chngcntr}
\usepackage{soul}
\usepackage{ amssymb }
\usepackage[misc]{ifsym}
\usepackage{tikz}
\usepackage{pgfplots}
\usepgfplotslibrary{dateplot}
\usepgfplotslibrary{groupplots}
\usepackage[ruled,linesnumbered,lined,commentsnumbered]{algorithm2e}
\usetikzlibrary{automata,positioning,arrows.meta,math,external}
\usetikzlibrary{decorations.pathreplacing}
\usetikzlibrary{shapes,shapes.geometric, snakes}
\usetikzlibrary{arrows, chains, fit, quotes}

\newcommand{\E}[1]{\mathbb{E}\left(#1\right)}
\newcommand{\inteval}[3]{\left[#1\right]^{#3}_{#2}}

\newtheorem{theorem}{Theorem}
\newtheorem{definition}{Definition}
\newtheorem{proof}{Proof}
\newtheorem{lemma}{Lemma}
\usepackage{pgfplots}
\usetikzlibrary{arrows.meta}

\pgfplotsset{compat=newest,
    width=6.5cm,
    height=5cm,
    scale only axis=true,
    max space between ticks=25pt,
    try min ticks=5
}
\tikzset{
    semithick/.style={line width=0.8pt},
}
\usepgfplotslibrary{groupplots}
\usepgfplotslibrary{dateplot}

\usepackage[ruled,linesnumbered,lined,commentsnumbered]{algorithm2e}
\usepackage{cancel}
\usepackage{eucal}
\usepackage{amssymb}
\usepackage[normalem]{ulem}
\usepackage{arxiv} 
\usepackage{ifsym}

\title{Expected Scalarised Returns Dominance: A New Solution Concept for Multi-Objective Decision Making \thanks{An earlier version of this work was presented at the Adaptive and Learning Agents Workshop 2021 \cite{hayes2021esr_set}. This article extends our workshop paper with additional theoretical analysis and new empirical results.}}

\author[1, **]{\textbf{Conor F. Hayes}}
\author[2]{\textbf{Timothy Verstraeten}}
\author[2, 3]{\textbf{Diederik M. Roijers}}
\author[1]{\textbf{Enda Howley}}
\author[1]{\textbf{Patrick Mannion}}
\affil[1]{National University of Ireland Galway, Ireland}
\affil[2]{Vrije Universiteit Brussel, Belgium}
\affil[3]{HU University of Applied Science Utrecht, The Netherlands}
\affil[**]{c.hayes13@nuigalway.ie}

\begin{document}

\maketitle

\begin{abstract}
In many real-world scenarios, the utility of a user is derived from a single execution of a policy. In this case, to apply multi-objective reinforcement learning, the expected utility of the returns must be optimised. Various scenarios exist where a user's preferences over objectives (also known as the utility function) are unknown or difficult to specify. In such scenarios, a set of optimal policies must be learned. However, settings where the expected utility must be maximised have been largely overlooked by the multi-objective reinforcement learning community and, as a consequence, a set of optimal solutions has yet to be defined. In this work we propose first-order stochastic dominance as a criterion to build solution sets to maximise expected utility. We also define a new dominance criterion, known as expected scalarised returns (ESR) dominance, that extends first-order stochastic dominance to allow a set of optimal policies to be learned in practice. Additionally, we define a new solution concept called the ESR set, which is a set of policies that are ESR dominant. Finally, we present a new multi-objective tabular distributional reinforcement learning (MOTDRL) algorithm to learn the ESR set in multi-objective multi-armed bandit settings.  
\end{abstract}
\section{Introduction}
When making decisions in the real world, decision makers must make trade-offs between multiple, often conflicting, objectives \cite{vamplew2021scalar}. In many real-world settings, a policy is only executed once. For example, consider a municipality that receives the majority of its electricity from local solar farms. To deal with the intermittency of the solar farms, the municipality wants to build a new electricity generation facility. The municipality are considering two choices: building a natural gas facility or adding a lithium-ion battery storage facility to the solar farms. Moreover, the municipality want to minimise $CO_{2}$ emissions while ensuring energy demand can continuously be met. Given a new energy generation facility will only be constructed once, a full distribution over each potential outcome for capacity to meet electricity demand and $CO_{2}$ emissions must be considered to make an optimal decision. The current state-of-the-art multi-objective reinforcement learning (MORL) literature focuses almost exclusively on learning polices that are optimal over multiple executions. Given such problems are salient, to fully utilise MORL in the real world, we must develop algorithms to compute a policy, or set of policies, that are optimal given the single-execution nature of the problem.

In multi-objective reinforcement learning (MORL) a user's preferences over objectives are represented by a utility function. In certain scenarios a user's preferences over objectives may be unknown; therefore, the utility function is unknown. In this case, a user is said to be in the unknown utility function or unknown weights scenario \cite{roijers2013survey}. The unknown utility function scenario has three phases: the learning phase, the selection phase and the execution phase. During the learning phase a multi-objective method \cite{vamplew2011evaluation_methods} is used compute a set of optimal policies and the set of policies is returned to the user. During the selection phase the utility function of the user becomes known and a policy from the computed set is selected which best reflects their preferences. The selected policy is then executed during the execution phase \cite{hayes2021practical}.

In contrast to single-objective reinforcement learning (RL), multiple optimality criteria exist for MORL \cite{roijers2013survey}. In scenarios where the utility of the user is derived from multiple executions of a policy, the scalarised expected returns (SER) must be optimised. However, in scenarios where the utility of a user is derived from a single execution of a policy, the expected utility of the returns (or expected scalarised returns, ESR) must be optimised. The majority of MORL research focuses on the SER criterion and linear utility functions \cite{radulescu2020survey}, which limits the applicability of MORL to real-world problems. In the real world, a user's utility function may be derived in a linear or non-linear manner. For known linear utility functions, single-objective methods can be used to learn an optimal policy \cite{roijers2013survey}. Non-linear utility functions do not distribute across the sums of the immediate and future returns, which invalidates the Bellman equation \cite{roijers2018multi}. Therefore, to learn optimal policies for non-linear utility functions, strictly multi-objective methods must be used.

For non-linear utility functions, the policies learned under the SER criterion and the ESR criterion can be different \cite{radulescu2020survey,radulescu2020utility}. The ESR criterion has received very little attention, to date, from the MORL community with some exceptions \cite{malerba2021esr,vamplew2021stocashtic,reymond21esr,roijers2018multi}. To learn optimal policies in many real-world scenarios where a policy will be executed only once, the ESR criterion must be optimised. For example, in a medical setting where a user has one opportunity to select a treatment, a user will aim to maximise the expected utility of a single outcome. However, choosing the wrong optimisation criterion (SER) for such a scenario could potentially lead to a different policy than that which would be learned under ESR. In the real world, like in the aforementioned scenario, learning a sub-optimal policy could have catastrophic outcomes. 

Therefore, it is crucial that the MORL community focuses on developing multi-objective algorithms that can learn optimal policies under the ESR criterion. Recently, a number of multi-objective methods have been implemented that can learn a single optimal policy under the ESR criterion \cite{roijers2018multi,hayes2021dmcts}. 
However, in the current MORL literature, no multi-policy algorithms exist for the ESR criterion. In fact, a set of optimal policies for the ESR criterion has yet to be defined.

Due to the lack of existing research for the ESR criterion, a formal definition of the requirements to learn optimal policies under the ESR criterion has yet to be determined. In Section \ref{sec:esr}, we define the requirements necessary to compute policies under the ESR criterion. The applicability of MORL to many real-world scenarios under the ESR criterion is limited because no solution set has been defined for scenarios when a user's utility function is unknown. In Section \ref{sec:stochastic_dominance_esr}, we show how first-order stochastic dominance can be used to define sets of optimal policies under the ESR criterion. In Section \ref{sec:esr_solution_sets}, we expand first-order stochastic dominance to define a new dominance criterion, called expected scalarised returns (ESR) dominance. This work proposes that ESR dominance can be used to compute a set of optimal policies, which we define as the \emph{ESR set}. Finally, we present a novel multi-objective tabular distributional reinforcement learning algorithm (MOTDRL) which aims to learn the \emph{ESR set} in scenarios when the utility function of the user is unknown. We apply MOTDRL to two different multi-objective multi-armed bandit settings where MOTDRL is able to learn the \emph{ESR set} in both settings. 


\section{Background}
In this section we introduce necessary background material, including multi-objective reinforcement learning, utility functions, the unknown utility function scenario, multi-objective multi-armed bandits, commonly used optimality criteria in multi-objective decision making, and stochastic dominance.

\subsection{Multi-Objective Reinforcement Learning}
In multi-objective reinforcement learning (MORL) \cite{hayes2021practical}, we deal with decision making problems with multiple objectives, often modelled as a multi-objective Markov decision process (MOMDP). An MOMDP is a tuple, $\mathcal{M} = (\mathcal{S}, \mathcal{A}, \mathcal{T}, \gamma, \mathcal{R})$, where $\mathcal{S}$ and $\mathcal{A}$ are the state and action spaces, $\mathcal{T} \colon \mathcal{S} \times \mathcal{A} \times \mathcal{S} \to \left[ 0, 1 \right]$ is a probabilistic transition function, $\gamma$ is a discount factor determining the importance of future rewards and $\mathcal{R} \colon \mathcal{S} \times \mathcal{A} \times \mathcal{S} \to \mathbb{R}^n$ is an $n$-dimensional vector-valued immediate reward function. In multi-objective reinforcement learning, $n>1$.

\subsection{Utility Functions}
In MORL, utility functions are used to model a user's preferences. In this work, utility functions map vector returns to a scalar value which represents the user's preferences over the returns,
\begin{equation}
    u : \mathbb{R}^{n} \rightarrow \mathbb{R},
\end{equation}
where $u$ is a utility function and $\textbf{R}^{n}$ is an n-dimensional vector. 
Linear utility functions are widely used to represent a user's preferences,
\begin{equation}
    \label{eqn:linear_utilityfunction}
    u = \sum_{i=1}^{n} w_{i}r_{i},
\end{equation}
where $w_{i}$ is the preference weight and $r_{i}$ is the value at position $i$ of the return vector. However, certain scenarios exist where linear utility functions cannot accurately represent a user's preferences. In this case, the user's preferences must be represented using a non-linear utility function.

In this paper, we consider monotonically increasing utility functions \cite{roijers2013survey}, i.e.,
\begin{equation}
 (\forall \, i, V_{i}^{\pi} \geq V_{i}^{\pi'} \wedge \exists \, i, V_{i}^{\pi} > V_{i}^{\pi'}) \implies (\forall \, u, u(\textbf{V}^{\pi}) > u(\textbf{V}^{\pi'})),   
\end{equation}

\noindent where $\mathbf{V}^{\pi}$ and $\mathbf{V}^{\pi'}$ are the values of executing policies $\pi$ and $\pi'$ respectively.

It is important to note, a monotonically increasing utility function also includes linear utility functions of the form in Equation \ref{eqn:linear_utilityfunction}. In certain scenarios the utility function may be unknown, therefore we do not know the shape of the utility function. If we assume the utility function is monotonically increasing we know that, if the value of one of the objectives in the return vector increases, then the utility also increases \cite{roijers2013survey}. This assumption makes it possible to determine a partial ordering over policies when the shape of the utility function is unknown. In this work we make no assumptions about the shape of the utility function but rather we assume the utility function is monotonically increasing.

\subsection{The Unknown Utility Function Scenario}
In MORL, a user's preferences over objectives can be modelled as a utility function \cite{roijers2013survey}. However, a user's utility function is often unknown at the time of learning or planning. In the taxonomy of multi-objective decision making (MODeM), this is known as the unknown utility function scenario (see Figure \ref{fig:unknown_utility_function_scenario}), where a set of optimal policies must be computed and returned to the user \cite{hayes2021practical}. In the unknown utility function scenario there are three phases: the learning or planning phase, the selection phase, and the execution phase. In the learning or planning phase a multi-objective planning or learning algorithm is deployed in a MOMDP. Given the utility function is unknown, the MORL algorithm computes a set of optimal policies during the learning or planning. During the selection phase, the user's preferences over objectives becomes known and the user selects a policy from the set of optimal policies that best reflects their preferences. Finally, during the execution phase the selected policy is executed.

\begin{figure}
\centering

\begin{tikzpicture}[>={Stealth[width=6pt,length=9pt]}, skip/.style={draw=none}, shorten >=1pt, accepting/.style={inner sep=1pt}, auto]
\draw (-80.0pt,8.0pt) node[below](0) {\begin{tabular}{c} MOMDP \end{tabular}} ;

\draw (-15.0pt, 0.0pt) node[rectangle, thick, fill=gray!20, minimum height=1cm,minimum width=0.5cm, draw, label=below:\begin{tabular}{c} \emph{planning or} \\ \emph{learning phase} \end{tabular}](1){$\text{algorithm}$};

\draw (45.0pt,14.5pt) node[below, minimum height=1cm,minimum width=0.5cm](2){\begin{tabular}{c} solution \\ set \end{tabular}} ;

\draw (90.0pt,50.0pt) node[below, minimum height=1cm,minimum width=0.5cm](5){\begin{tabular}{c} utility \\ function \end{tabular}} ;

\draw (90.0pt,0.0pt) node[below, minimum height=1cm,minimum width=0.5cm](6){} ;

\draw (130.0pt, 0.0pt) node[rectangle, thick, fill=gray!20, minimum height=1cm,minimum width=0.5cm, draw, label=below:$\text{\emph{selection phase}}$](3){\begin{tabular}{c} user \\ selection \end{tabular}};

\draw (210.0pt,14.5pt) node[below, minimum height=1cm,minimum width=0.5cm, label=below:\begin{tabular}{c} \emph{execution} \\ \emph{phase}\end{tabular}](4){\begin{tabular}{c} single \\ solution \end{tabular}} ;

\draw [dashed] (75.0pt,25pt) -- (75.0pt,-40.0pt);
\draw [dashed] (170.0pt,25pt) -- (170.0pt,-40.0pt);

\path[->] (0) edge node{} (1);
\path[->] (1) edge node{} (2);
\path[->] (2) edge node{} (3);
\path[->] (3) edge node{} (4);
\path[->] (5) edge node{} (6);
\end{tikzpicture}
\caption{The unknown utility function scenario \cite{hayes2021practical}.}
\label{fig:unknown_utility_function_scenario}
\end{figure}
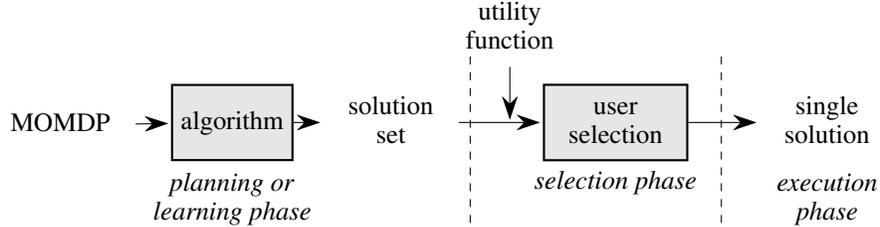

\subsection{Multi-Objective Multi-Armed Bandits}
Multi-objective multi-armed bandits (MOMAB) \cite{drugan2013momab} are a natural extension of multi-armed bandits, where each arm returns an n-dimensional reward vector $\textbf{R}^{n}$, where n is the number of objectives. At each timestep, $t$, the agent must select an arm, $i$, and receives a reward vector. The returns in an MOMAB setting can be deterministic \cite{drugan2013momab} or stochastic \cite{auer2016}. Many algorithms focus on the MOMAB setting and learn a set of arms that are optimal \cite{drugan2013momab,oner2018combinatorial,roijers2017interactive,yahyaa2015Thompson}. 

For example, Pareto UCB-1 \cite{drugan2013momab} is an algorithm that can learn a set of optimal policies in an MOMAB setting. Pareto UCB-1 \cite{drugan2013momab} initially selects each arm once, then at each timestep the algorithm computes the mean vector of each of the multi-objective arms and adds the upper confidence bound to the mean return vector. Using this method Pareto UCB-1, can learn the Pareto front in an MOMAB setting.

\subsection{Scalarised Expected Returns and Expected Scalarised Returns}
\label{sec:esr_versus_ser}
For MORL, the ability to express a user's preferences over objectives as a utility function is essential when learning a single optimal policy. In MORL, different optimality criteria exist \cite{roijers2013survey}.  
Additionally, the utility function can be applied to the expectation of the returns, or the utility function can be applied directly to the returns before computing the expectation. Calculating the expected value of the return of a policy before applying the utility function leads to the scalarised expected returns (SER) optimisation criterion: 

\begin{equation}
    V_{u}^{\pi} = u\left(\mathbb{E} \left[ \sum\limits^\infty_{t=0} \gamma^t {\textbf{r}}_t \:\middle|\: \pi, \mu_0 \right]\right),
    \label{eqn:ser}
\end{equation}
where $\mu_0$ is the probability distribution over possible starting states.

SER is the most commonly used criterion in the multi-objective (single agent) planning and reinforcement learning literature \cite{roijers2013survey}. For SER, a coverage set is defined as a set of optimal solutions for all possible utility functions. If the utility function is instead applied to the returns before computing the expectation, this leads to the expected scalarised returns (ESR) optimisation criterion \cite{roijers2018multi,hayes2021dmcts,roijers2013survey}: 
\begin{equation}
    V_{u}^{\pi} = \mathbb{E} \left[ u\left( \sum\limits^\infty_{t=0} \gamma^t {\textbf{r}}_t \right) \:\middle|\: \pi, \mu_0 \right].
    \label{eqn:esr}
\end{equation}
ESR is the most commonly used criterion in the game theory literature on multi-objective games~\cite{radulescu2020survey}.

\subsection{Stochastic Dominance}
\label{sec:stochastic_dominance}
Stochastic dominance \cite{hadar1969,bawa1975} gives a partial order between distributions and can be used when making decisions under uncertainty (see Figure \ref{fig:esr_cdf}). Stochastic dominance is particularly useful when a distribution must be taken into consideration rather than an expected value when making decisions. Stochastic dominance is a prominent dominance criterion in finance, economics and decision theory. When making decisions under uncertainty, stochastic dominance can be used to determine the most risk averse decision. Various degrees of stochastic dominance exist, however, in this paper we focus on first-order stochastic dominance (FSD). FSD can be used to give a partial ordering over random variables or random vectors to give an FSD dominant set.

\begin{figure}
\centering
    \includegraphics[height = 7cm, width=9cm]{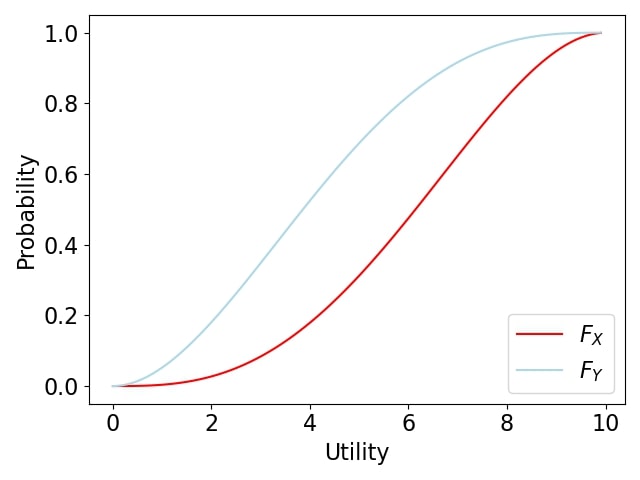}
    \caption{For random variables $X$ and $Y$, $X$ $\succeq_{FSD}$ $Y$, where $F_{X}$ and $F_{Y}$ are the cumulative distribution functions (CDFs) of $X$ and $Y$ respectively. In this case, $X$ is preferable to $Y$ because higher utilities occur with greater frequency in $F_{X}$.}
    \label{fig:esr_cdf}
\end{figure}

In Definition \ref{definition:sd_conditions} we present the necessary conditions for FSD and in Theorem \ref{theorem:fsd_expectedval} we prove that if a random variable is FSD dominant it has at least as high an expected value as another random variable \cite{wolfstetter_1999}. We use the work of Wolfstetter \cite{wolfstetter_1999} to prove Theorem \ref{theorem:fsd_expectedval}.

\begin{definition}
\label{definition:sd_conditions}
For random variables X and Y, X $\succeq_{FSD}$ Y if:
\[
P(X > z) \geq P(Y > z), \forall \, z
\]
\end{definition}
If we consider the cumulative distribution function (CDF) of X, $F_{X}$, and the CDF of Y, $F_{Y}$, we can say that X $\succeq_{FSD}$ Y if:
\[
F_{X}(z) \leq F_{Y}(z), \forall \, z.
\]

\begin{theorem}
\label{theorem:fsd_expectedval}
If X $\succeq_{FSD}$ Y, then X has a greater than or equal expected value as Y.
\[
X \succeq_{FSD} Y \implies E(X) \geq E(Y).
\]
\end{theorem}
\begin{proof}
By a known property of expected values the following is true for any random variable:
\[
\mathbb{E}(X) = \int_{0}^{+\infty} (1 - F_{X}(x)) \,dx
\]
\[
\mathbb{E}(Y) =  \int_{0}^{+\infty} (1 - F_{Y}(x)) \,dx
\]

Therefore, if X $\succeq_{FSD}$ Y then:
\[
\int_{0}^{+\infty} (1 - F_{X}(x)) \,dx \geq \int_{0}^{+\infty} (1 - F_{Y}(x)) \,dx
\]

Which gives,
\[
\mathbb{E}(X) \geq \mathbb{E}(Y).
\]
\cite{wolfstetter_1999} \end{proof}

\section{Expected Scalarised Returns}
\label{sec:esr}
In contrast to single-objective reinforcement learning, different optimality criteria exist for MORL. In scenarios where the utility of a user is derived from multiple executions of a policy, the agent should optimise over the scalarised expected returns (SER) criterion. In scenarios where the utility of a user is derived from a single execution of a policy, the agent should optimise over the expected scalarised returns (ESR) criterion. Let us consider, as an example, a power plant that generates electricity for a city and emits harmful $CO_2$ and greenhouse gases. City regulations have been imposed which limit the amount of pollution that the power plant can generate. If the regulations require that the emissions from the power plant do not exceed a certain amount over an entire year, the SER criterion should be optimised. In this scenario, the regulations allow for the pollution to vary day to day, as long as the emissions do not exceed the regulated level for a given year. However, if the regulations are much stricter and the power plant is fined every day it exceeds a certain level of pollution, it is beneficial to optimise under the ESR criterion.

The majority of MORL research focuses on linear utility functions. However, in the real world, a user's utility function can be non-linear. For example, a utility function is non-linear in situations where a minimum value must be achieved on each objective \cite{ocallaghan2021exploring}. Focusing on linear utility functions limits the applicability of MORL in real-world decision making problems. For example, linear utility functions cannot be used to learn policies in concave regions of the Pareto front \cite{vamplew2008limitations}. Furthermore, if a user's preferences are non-linear, these are fundamentally incompatible with linear utility functions. In this case, strictly multi-objective methods must be used to learn optimal policies for non-linear utility functions. In MORL, for non-linear utility functions, different policies are preferred when optimising under the ESR criterion versus the SER criterion \cite{radulescu2020utility}. It is important to note that, for linear utility functions, the distinction between ESR and SER does not exist \cite{radulescu2020survey}.

For example, a decision maker has to choose between the following lotteries, $L_{1}$ and $L_{2}$, which are highlighted in Table \ref{table:esr_ser_lotteries}.

\begin{table}[h]
    \centering
    \begin{tabular}{| c | c |}
    \multicolumn{2}{c}{$L_{1}$} \\
    \hline
         P($L_{1}$= $\textbf{R}$) & \textbf{R} \\
    \hline
         0.5 & (4, 3) \\
    \hline 
          0.5 & (2, 3) \\
    \hline
    \end{tabular}
    \quad
    \begin{tabular}{| c | c |}
    \multicolumn{2}{c}{$L_{2}$} \\
    \hline
         P($L_{2}$=$\textbf{R}$) & \textbf{R} \\
    \hline
         0.9 & (1, 3) \\
    \hline 
          0.1 & (10, 2) \\
    \hline
    \end{tabular}
    \caption{A lottery, $L_{1}$, has two possible returns, (4, 3) and (2, 3), each with a probability of 0.5. A lottery, $L_{2}$, has two possible returns, (1, 3) with a probability of 0.9 and (10, 2) with a probability of 0.1.}
    \label{table:esr_ser_lotteries}
\end{table}
The decision maker has the following non-linear utility function:
\begin{equation}
\label{eqn:esr_ser_utility_function}
    u(\textbf{x}) = x_{1}^2 + x_{2}^2,
\end{equation}
\noindent where $\textbf{x}$ is a vector returned from $\mathbf{R}$ in Table \ref{table:esr_ser_lotteries}, and $x_{1}$ and $x_{2}$ are the values of two objectives. Note that this utility function is monotonically increasing for $x_{1} \geq 0$ and $x_{2} \geq 0$.
Under the SER criterion, the decision maker will compute the expected value of each lottery, apply the utility function, and select the lottery that maximises their utility function. Let us consider which lottery the decision maker will play under the SER criterion:

\[
L_{1}: \,E(L_{1}) = 0.5(4, 3) + 0.5(2, 3) = (2, 1.5) + (1, 1.5) = (3, 3)
\]
\[
L_{1}: \,u(E(L_{1})) = (3^{2} + 3^{2}) = 9 + 9 = 18
\]
\[
L_{2}: \,E(L_{2}) = 0.9(1, 3) + 0.1(10, 2) = (0.9, 2.7) + (1, 0.2) = (1.9, 2.9)
\]
\[
L_{2}: \,u(E(L_{2})) = (1.9^{2} + 2.9^{2}) = 3.61 + 8.41 = 12.02
\]
Therefore, a decision maker with the utility function in Equation \ref{eqn:esr_ser_utility_function} will prefer to play lottery $L_{1}$ under the SER criterion.

Under the ESR criterion, the decision maker will first apply the utility function to the return vectors, compute the expectation, and select the lottery to maximise their utility function. Let us consider how a decision maker will choose which lottery to play under the ESR criterion:
\[
L_{1}: \,\mathbb{E}(u(L_{1})) = 0.5(u(4, 3)) + 0.5(u(2, 3)) = 0.5(4^{2} +3^{2}) + 0.5(2^{2} + 3^{2})
\]
\[
= 0.5(25) + 0.5(13) = 12.5 + 6.5 = 19
\]
\[
L_{2}: \,\mathbb{E}(u(L_{2})) = 0.9(u(1, 3)) + 0.1(u(10, 2)) = 0.9(1^{2} + 3^{2}) + 0.1(10^{2} + 2^{2}) 
\]
\[
= 0.9(10) + 0.1(104) = 9 + 10.4 = 19.4 
\]
Therefore, a decision maker with the utility function in Equation \ref{eqn:esr_ser_utility_function} will prefer to play lottery $L_{2}$  under the ESR criterion. From the example, it is clear that users with the same non-linear utility function can prefer different policies, depending on which multi-objective optimisation criterion is selected. Therefore, it is critical that the distinction ESR and SER is taken into consideration when selecting a MORL algorithm to learn optimal policies in a given scenario. 
The majority of MORL research focuses on the SER criterion \cite{radulescu2020survey}. By comparison, the ESR criterion has received very little attention from the MORL community \cite{roijers2013survey,hayes2021dmcts,roijers2018multi,radulescu2020survey}. Many of the traditional MORL methods cannot be used when optimising under the ESR criterion, given non-linear utility functions in MOMDPs do not distribute across the sum of immediate and future returns which invalidates the Bellman equation \cite{roijers2018multi},

\begin{equation}
\begin{split}
\max_\pi\ &\mathbb{E} \left[u\left({\bf R}_t^- + \sum_{i=t}^{\infty} \gamma^i {\bf r}_i\right)\ \middle|\ \pi, s_t \right] \not= \\
&u({\bf R}_t^-) + \max_\pi \mathbb{E}\left[  u\left( \sum_{i=t}^{\infty} \gamma^i {\bf r}_i\right)\ \middle|\ \pi, s_t \right],
\end{split}
\end{equation}
where $u$ is a non-linear utility function and $\bf{R^{-}_{t}}$ $=$ $\sum_{i=0}^{t - 1} \gamma^i {\bf r}_i$.

An example of an algorithm that can learn policies for non-linear utility functions and the ESR criterion is distributional Monte Carlo tree search (DMCTS) \cite{hayes2021dmcts}. Hayes et al. \cite{hayes2021dmcts} use Monte Carlo tree search to calculate the returns of a full policy and compute a posterior distribution over the expected utility of individual policy executions. DMCTS achieves state of the art performance under the ESR criterion. Hayes et al. \cite{hayes2021dmcts} demonstrate that, when optimising under the ESR criterion, making decisions based on a distribution over the utility of the returns is particularly useful when learning in realistic problems where rewards are stochastic. 

However, DMCTS and other MORL algorithms that optimise for the ESR criterion \cite{roijers2013survey,roijers2018multi,malerba2021esr} require the utility function of a user to be known a priori. In practice, many scenarios exist where a user's utility function may be unknown at the time of learning or planning. To compute policies under the ESR criterion when a user's utility function is unknown, a distribution over the returns must be maintained. To highlight why a distribution over the returns must be used when the utility function of a user is unknown, let us consider the following example in Table~\ref{table:distribution_esr_lotteries}.\footnote{Generally, in the unknown utility function scenario a set of optimal policies is calculated. Under the ESR criterion, a set of optimal policies has yet be defined. Therefore, this example does not calculate a set of optimal policies but instead illustrates why a distribution over the returns is required under the ESR criterion. We define a set of optimal policies under the ESR criterion in a later section.}

\begin{table}

    \centering
    \begin{tabular}{| c | c |}
    \multicolumn{2}{c}{$L_{3}$} \\
    \hline
         P($L_{3}$= $\textbf{R}$) & \textbf{R} \\
    \hline
         0.5 & (-20, 1) \\
    \hline 
          0.5 & (20, 3) \\
    \hline
    \end{tabular}
    \quad
    \begin{tabular}{| c | c |}
    \multicolumn{2}{c}{$L_{4}$} \\
    \hline
         P($L_{4}$=$\textbf{R}$) & \textbf{R} \\
    \hline
         0.9 & (0, 2) \\
    \hline 
          0.1 & (5, 2) \\
    \hline
    \end{tabular}
    \caption{A lottery, $L_{3}$, has two possible returns, (-20, 1) and (20, 3), each with a probability of 0.5. A lottery, $L_{4}$, has two possible returns, (0, 2) with a probability of 0.9 and (5, 2) with a probability of 0.1.}
    \label{table:distribution_esr_lotteries}
\end{table}

To determine which lottery to play while optimising for the ESR criterion, the utility function must first be applied, then the expected utility can be computed (see Equation \ref{eqn:esr}):

\[
u(L_{3}) = u((-20, 1)) + u((20, 3))
\]
\[
\mathbb{E}(u(L_{3})) = 0.5(u((-20, 1))) + 0.5(u((20, 3)))
\]
\[
u(L_{4}) = u((0, 2)) + u((5, 2))
\]
\[
\mathbb{E}(u(L_{4})) = 0.9(u((0, 2))) + 0.1(u((5, 2)))
\]
Given the utility function is unknown, it impossible to compute the expected utility. Moreover, a distribution over the returns received from a policy execution must be maintained in order to optimise for the ESR criterion. Maintaining a distribution over the returns ensures the expected utility can be computed once the user's utility function becomes known at decision time.


As demonstrated above, maintaining a distribution over the returns is critical to learning optimal policies when the utility function of a user is unknown. Therefore, to compute a set of optimal policies under the ESR criterion it is necessary to adopt a distributional approach.


To adopt a distributional approach to multi-objective decision making, we must first introduce a multi-objective version of the return distribution \cite{bellemare2017distributional} \footnote{Bellemare et al. \cite{bellemare2017distributional} introduce a value distribution. However given the distribution is a distribution over the returns, not values, we prefer the term return distribution.}, $\textbf{Z}^{\pi}$. A return distribution, $\textbf{Z}^{\pi}$, is equivalent to a multivariate distribution where a dimension exists per objective. The return distribution, $\textbf{Z}^{\pi}$, gives the distribution over returns of a random vector \cite{sutton1998rl} when a policy $\pi$ is executed, such that,

\begin{equation}
\mathbb{E} \, \textbf{Z}^{\pi} =  \mathbb{E} \left[  \sum\limits^\infty_{t=0} \gamma^t {\textbf{r}}_t  \:\middle|\: \pi, \mu_0 \right].
\end{equation}
Moreover, a return distribution can be used to represent policies. Under the ESR criterion, the utility-of-the-return-distribution, $Z_{u}^{\pi}$, is defined as a distribution over the scalar utilities received from applying the utility function to each vector in the return distribution, $\textbf{Z}^{\pi}$. Therefore, $Z_{u}^{\pi}$ is a distribution over the scalar utility of vector returns of a random vector received from executing a policy $\pi$, such that,
\begin{equation}
\mathbb{E} \, Z_{u}^{\pi} =  \mathbb{E} \left[ u\left( \sum\limits^\infty_{t=0} \gamma^t {\textbf{r}}_t \right) \:\middle|\: \pi, \mu_0 \right].
\end{equation}
The utility-of-the-return-distribution can only be calculated when the utility function is known a priori.

When the utility function of a user is unknown, a set of policies that are optimal for all monotonically increasing utility functions must be learned. However, for the ESR criterion, a set of optimal solutions has yet to be defined. To learn a set of optimal policies under the ESR criterion we must develop new methods. In Section \ref{sec:stochastic_dominance_esr} we apply first-order stochastic dominance to determine a partial ordering over return distributions under the ESR criterion.

\section{Stochastic Dominance for ESR}
\label{sec:stochastic_dominance_esr}
For MORL there are two classes of algorithms: single-policy and multi-policy algorithms \cite{vamplew2011evaluation_methods,roijers2013survey}. When the user's utility function is known a priori, it is possible to use a single-policy algorithm \cite{hayes2021dmcts,roijers2018multi} to learn an optimal solution. However, when the user's utility function is unknown we aim to learn a set of policies that are optimal for all monotonically increasing utility functions. The current literature on the ESR criterion focuses only on scenarios where the utility function of a user is known \cite{hayes2021dmcts,roijers2018multi}, overlooking scenarios where the utility function of a user is unknown. Moreover, a set of solutions under the ESR criterion for the unknown utility function scenario \cite{roijers2013survey} has yet to be defined.

Various algorithms have been proposed to learn solution sets under the SER criterion (see Section \ref{sec:esr_versus_ser}), for example \cite{wang2012multi,vanmoffaert2014multiobjective,roijer2014ols}. Under the SER criterion, multi-policy algorithms determine optimality by comparing policies based on the utility of expected value vectors (see Equation \ref{eqn:ser}). In contrast, under the ESR criterion it is crucial to maintain a distribution over the utility of possible vector-valued outcomes. SER multi-policy algorithms cannot be used to learn optimal policies under the ESR criterion because they compute expected value vectors. It is necessary to develop new methods that can generate solution sets for the ESR criterion with unknown utilities. The development of methods that determine an optimal partial ordering over return distributions is a promising avenue to address this challenge.

First-order stochastic dominance (see Section \ref{sec:stochastic_dominance}) is a method which gives a partial ordering over random variables \cite{wolfstetter_1999,levy1992survey}. FSD compares the cumulative distribution functions (CDFs) of the underlying probability distributions of random variables to determine optimality. When computing policies under the ESR criterion, it is essential that the expected utility is maximised. To use FSD for the ESR criterion, we must show the FSD conditions presented in Section \ref{sec:stochastic_dominance} also hold when optimising the expected utility for unknown monotonically increasing utility functions. 

For the single-objective case, Theorem \ref{theorem:utility_fsd} proves for random variables X and Y, if X $\succeq_{FSD}$ Y, the expected utility of X is greater than, or equal to, the expected utility of Y for monotonically increasing utility functions. In Theorem \ref{theorem:utility_fsd}, random variables X and Y are considered, and their corresponding CDFs $F_{X}$, $F_{Y}$. The work of Mas-Colell et al. \cite{mas1995microeconomic} is used as a foundation for Theorem \ref{theorem:utility_fsd}.

\begin{theorem}
\label{theorem:utility_fsd}
A random variable, X, is preferred to a random variable, Y, for all decision makers with a monotonically increasing utility function if, X $\succeq_{FSD}$ Y. 
\[
X \succeq_{FSD} Y \implies \mathbb{E}(u(X)) \geq \mathbb{E}(u(Y))
\]
\end{theorem}
\begin{proof}
If X $\succeq_{FSD}$ Y, then\footnote{CDFs with lower probability values for a given $z$ are preferable. Figure \ref{fig:esr_cdf} explains why this is the case.},
\[
F_{X}(z) \leq F_{Y}(z), \forall \, z
\]
Since, 
\[
\mathbb{E}(u(X)) =  \int_{-\infty}^{\infty} u(z) dF_{X}(z)
\]
\[
\mathbb{E}(u(Y)) = \int_{-\infty}^{\infty} u(z) dF_{Y}(z)
\]
When integrating both $\mathbb{E}(u(X))$ and $\mathbb{E}(u(Y))$ by parts, the following results is generated:
\[
\mathbb{E}(u(X)) = [u(z)F_{X}(z)]_{-\infty}^{\infty} - \int_{-\infty}^{\infty} u'(z)F_{X}(z) \,dz
\]
\[
\mathbb{E}(u(Y)) = [u(z)F_{Y}(z)]_{-\infty}^{\infty} - \int_{-\infty}^{\infty} u'(z)F_{Y}(z) \,dz
\]
Given $F_X(-\infty)$ = $F_Y(-\infty)$ = 0 and $F_X(\infty)$ = $F_Y(\infty)$ = 1, the first terms in $\mathbb{E}(u(X))$ and $\mathbb{E}(u(Y))$ are equal, and thus
\[
\mathbb{E}(u(X)) -\mathbb{E}(u(Y)) = \int_{-\infty}^{\infty} u'(z)F_{Y}(z) \,dz - \int_{-\infty}^{\infty} u'(z)F_{X}(z) \,dz
\]
Since $F_{X}(z) \leq F_{Y}(z)$ and $u'(z) \geq 0$ for all monotonically increasing utility functions, then
\[
\mathbb{E}(u(X)) - \mathbb{E}(u(Y)) \geq 0
\]
and thus, 
\[
\mathbb{E}(u(X)) \geq \mathbb{E}(u(Y)) 
\]
\end{proof}
A utility function maps an input (scalar or vector return) to an output (scalar utility). Since the probability of receiving some utility is equal to the probability of receiving some return for a random variable, X, we can write the following:
\begin{equation}
\label{eqn:prob_fsd_utility}
P(X > c) = P(u(X) > u(c)),    
\end{equation}
where $c$ is a constant.
Using the results shown in Theorem \ref{theorem:utility_fsd} and Equation \ref{eqn:prob_fsd_utility}, the FSD conditions highlighted in Section \ref{sec:stochastic_dominance} can be rewritten to include monotonically increasing utility functions:
\begin{equation}
\label{eqn:prob_all_utility_FSD}
P(u(X) > u(z)) \geq P(u(Y) > u(z)) 
\end{equation}
\begin{definition}
Let X and Y be random variables. X dominates Y for all decision makers with a monotonically increasing utility function if the following is true:
\[
X \succeq_{FSD} Y \Leftrightarrow
\]
\[
\forall u : \forall {v} : P(u(X) > u({v})) \geq P( u(Y) > u({v})).
\]
\end{definition}

In MORL, the return from the reward function is a vector, where each element in the return vector represents an objective. To apply FSD to MORL under the ESR criterion, random vectors must be considered. In this case, a random vector (or multi-variate random variable) is a vector whose components are scalar-valued random variables on the same probability space. For simplicity, this paper focuses on the case in which a random vector has two random variables, known as the bi-variate case. FSD conditions have been proven to hold for random vectors with $n$ random variables in the works of Sriboonchitta et al. \cite{sriboonchitta2009mvfsd}, Levhari et al. \cite{levhari1975multivariatedistributions}, Nakayama et al. \cite{NAKAYAMA1981} and Scarsini \cite{scarsinimultivariate1988}. In Theorem \ref{theorem:multivariate_FSD}, the work of Atkinson and Bourguignon \cite{atkinson1982bivariate_fsd} is distilled into a suitable Theorem for MORL. Theorem \ref{theorem:multivariate_FSD} highlights how the conditions for FSD hold for random vectors when optimising under the ESR criterion for a monotonically increasing utility function, $u$, where $\frac{\partial^2 u(x_1, x_2)}{\partial x_1 \partial x_2} \le 0$ \cite{richard1975}. It is important to note Atikson and Bourguignon \cite{atkinson1982bivariate_fsd} have shown the conditions for FSD hold for random vectors for utility functions where $\frac{\partial^2 u(x_1, x_2)}{\partial x_1 \partial x_2} \ge 0$. We plan to extend these conditions for MORL in a future work. In Theorem \ref{theorem:multivariate_FSD}, $\textbf{X}$ and $\textbf{Y}$ are random vectors where each random vector consists of two random variables, $\textbf{X} = [X_{1}, X_{2}]$ and $\textbf{Y} = [Y_{1}, Y_{2}]$. $F_{{X_{1}X{_2}}}$ and $F_{Y_{1}Y_{2}}$ are the corresponding CDFs.

\begin{theorem}
\label{theorem:multivariate_FSD}
Assume that $u\ :\ \mathbb{R}_{\ge 0} \times \mathbb{R}_{\ge 0} \rightarrow \mathbb{R}_{\ge 0}$ is a monotonically increasing function, with $\frac{\partial u(x_1, x_2)}{\partial x_1} \ge 0$, $\frac{\partial u(x_1, x_2)}{\partial x_2} \ge 0$ and $\frac{\partial^2 u(x_1, x_2)}{\partial x_1 \partial x_2} \le 0$. If, for random vectors $\mathbf{X}$ and $\mathbf{Y}$, $\mathbf{X} \succeq_{FSD} \mathbf{Y}$, then $\mathbf{X}$ is preferred to $\mathbf{Y}$ by all decision makers, i.e.,

\begin{equation*}
\mathbf{X} \succeq_{FSD} \mathbf{Y} \implies \mathbb{E}(u(\mathbf{X})) \geq \mathbb{E}(u(\mathbf{Y}))
\end{equation*}
\end{theorem}

\begin{proof}
As $\mathbf{X} \succeq_{FSD} \mathbf{Y}$, $\forall t, z$ we have
\begin{equation*}
\begin{split}
&F_\mathbf{X}(t, z) \leq F_\mathbf{Y}(t, z),\\
\text{or }& \Delta_F(t, z) = F_\mathbf{X}(t, z) - F_\mathbf{Y}(t, z) \le 0.
\end{split}
\end{equation*}
The expected utility of the random variable $\mathbf{X}$ can be written as follows: 
\begin{equation*}
\E{u(\textbf{X})} = \int^{\infty}_{0} \int^{\infty}_{0} u(t, z)f_\mathbf{X}(t, z)dtdz,
\end{equation*}
where $f$ is the probability density function of $\textbf{X}$. Note that
\begin{equation*}
\begin{split}
\Delta_f(t, z) &= f_\mathbf{X}(t, z) - f_\mathbf{Y}(t, z)\\
&= \frac{\partial^2 \Delta_F(t, z)}{\partial t \partial z}.
\end{split}
\end{equation*}

Using integration-by-parts (I), and the fact that $\Delta_F(t, 0) = \frac{\partial \Delta_F(0, z)}{\partial z} = 0$ (Z), we obtain:
\begin{equation*}
\begin{split}
&\E{u(\textbf{X})} - \E{u(\textbf{Y})}\\
&= \int^{\infty}_{0} \int^{\infty}_{0} u(t, z)\Delta_f(t, z)dtdz\\
&\stackrel{(I)}{=} \int^{\infty}_{0} \inteval{u(t, z)\frac{\partial \Delta_F(t, z)}{\partial z}}{t=0}{\infty} dz - \int^{\infty}_{0}\int^{\infty}_{0} \frac{\partial u(t, z)}{\partial t}\frac{\partial \Delta_F(t, z)}{\partial z} dtdz
\end{split}
\end{equation*}
\begin{equation*}
\begin{split}
&\stackrel{(I)}{=} \int^{\infty}_{0} \inteval{u(t, z)\frac{\partial \Delta_F(t, z)}{\partial z}}{t=0}{\infty} dz - \int^{\infty}_0\inteval{\frac{\partial u(t, z)}{\partial t}\Delta_F(t, z)}{z=0}{\infty} dt + \\ 
&\int^{\infty}_{0}\int^{\infty}_{0} \frac{\partial^2 u(t, z)}{\partial t\partial z}\Delta_F(t, z) dtdz\\
\end{split}
\end{equation*}
\begin{equation*}
\begin{split}
&\stackrel{(Z)}{=} \int^{\infty}_{0} \lim_{t \rightarrow \infty} u(t, z)\frac{\partial \Delta_F(t, z)}{\partial z} dz - \int^{\infty}_0\lim_{z \rightarrow \infty}\frac{\partial u(t, z)}{\partial t}\Delta_F(t, z) dt + \\ &\int^{\infty}_{0}\int^{\infty}_{0} \frac{\partial^2 u(t, z)}{\partial t\partial z}\Delta_F(t, z) dtdz.
\end{split}
\end{equation*}

Given that $\frac{\partial^2 u(t, z)}{\partial t\partial z} \le 0$, $\frac{\partial u(t, z)}{\partial t} \ge 0$ and $\Delta_F(t, z) \le 0$, we know that the last two terms are positive.
Therefore, we can state that
\begin{equation*}
\begin{split}
&\E{u(\textbf{X})} - \E{u(\textbf{Y})}\\
&= \int^{\infty}_{0} \lim_{t \rightarrow \infty} u(t, z)\frac{\partial \Delta_F(t, z)}{\partial z} dz - \int^{\infty}_0\lim_{z \rightarrow \infty}\frac{\partial u(t, z)}{\partial t}\Delta_F(t, z) dt + \\
\end{split}
\end{equation*}
\begin{equation*}
\begin{split}
&\int^{\infty}_{0}\int^{\infty}_{0} \frac{\partial^2 u(t, z)}{\partial t\partial z}\Delta_F(t, z) dtdz 
\ge \int^{\infty}_{0} \lim_{t \rightarrow \infty} u(t, z)\frac{\partial \Delta_F(t, z)}{\partial z} dz.
\end{split}
\end{equation*}

According to Lemma~\ref{lemma:monotone_convergence} (see Section \ref{sec:supplemental_material}), as $u(t, z) F(t, z)$ is a positive monotonically increasing function in both $t$ and $z$, we know that:
\begin{equation*}
\begin{split}
&\int^\infty_0 \lim_{t \rightarrow \infty} u(t, z)\frac{\partial F(t, z)}{\partial z} dz = \lim_{t \rightarrow \infty} \int^\infty_0 u(t, z)\frac{\partial F(t, z)}{\partial z} dz.\\
\end{split}
\end{equation*}

Using integration-by-parts (I), and the fact that $\Delta_F(t, 0) = 0$ (Z), we have:
\begin{equation*}
\begin{split}
&\E{u(\textbf{X})} - \E{u(\textbf{Y})}\\
&\ge \lim_{t \rightarrow \infty} \int^\infty_0 u(t, z)\frac{\partial \Delta_F(t, z)}{\partial z} dz\\
&\stackrel{(I)}{=} \lim_{t \rightarrow \infty} \inteval{u(t, z)\Delta_F(t, z)}{0}{\infty} - \lim_{t \rightarrow \infty} \int^\infty_0 \frac{\partial u(t, z)}{\partial z}\Delta_F(t, z) dz\\
&\stackrel{(Z)}{=} \lim_{t \rightarrow \infty} \lim_{z \rightarrow \infty} u(t, z)\Delta_F(t, z) - \lim_{t \rightarrow \infty} \int^\infty_0 \frac{\partial u(t, z)}{\partial z}\Delta_F(t, z) dz.\\
\end{split}
\end{equation*}

Finally, given that $\frac{\partial u(t, z)}{\partial z} \ge 0$ and $\Delta_F(t, z) \le 0$, we know that:
\begin{equation*}
\begin{split}
&\E{u(\textbf{X})} - \E{u(\textbf{Y})}\\
&\ge \lim_{t \rightarrow \infty} \lim_{z \rightarrow \infty} u(t, z)\Delta_F(t, z) - \lim_{t \rightarrow \infty} \int^\infty_0 \frac{\partial u(t, z)}{\partial z}\Delta_F(t, z) dz\\
&\ge 0
\end{split}
\end{equation*}
\end{proof}

Using the results from Theorem \ref{theorem:multivariate_FSD}, Equation \ref{eqn:prob_all_utility_FSD} can be updated to include random vectors,

\begin{equation}
P(u(\textbf{X}) > u(\textbf{z})) \geq P(u(\textbf{Y}) > u(\textbf{z})).
\end{equation}
\begin{definition}
\label{definition:multi_objective_fsd_esr}
For random vectors $\textbf{X}$ and $\textbf{Y}$, $\textbf{X}$ is preferred over $\textbf{Y}$ by all decision makers with a monotonically increasing utility function if, and only if, the following is true:
\[
\textbf{X} \succeq_{FSD} \textbf{Y} \Leftrightarrow
\]
\[
\forall u :  ( \forall {\textbf{v}}: P(u(\textbf{X}) > u({\textbf{v}})) \geq P( u(\textbf{Y}) > u({\textbf{v}})) 
\]
\end{definition}
Using the results from Theorem \ref{theorem:multivariate_FSD} and Definition \ref{definition:multi_objective_fsd_esr}, it is possible to extend FSD to MORL. 
For MORL, under the ESR criterion, the return distribution, $\textbf{Z}^{\pi}$, is considered to be the full distribution of the returns of a random vector received when executing a policy, $\pi$ (see Section \ref{sec:esr}), return distributions can be used to represent policies. In this case, it is possible to use FSD to obtain a partial ordering over policies. For example, consider two policies, $\pi$ and $\pi'$, where each policy has the underlying return distribution $\textbf{Z}^{\pi}$ and $\textbf{Z}^{\pi'}$. If $\textbf{Z}^{\pi} \, \succeq_{FSD} \textbf{Z}^{\pi'}$ then $\pi$ will be preferred over $\pi'$.
\begin{definition}
Policies $\pi$ and $\pi'$ have return distributions $\textbf{Z}^{\pi}$ and $\textbf{Z}^{\pi'}$. Policy $\pi$ is preferred over policy $\pi'$ by all decision makers with a utility function, $u$, that is monotonically increasing if, and only if, the following is true:
\[
\textbf{Z}^{\pi} \succeq_{FSD} \textbf{Z}^{\pi'}.
\]
\end{definition}
Now that a partial ordering over policies has been defined under the ESR criterion for the unknown utility function scenario, it is possible to define a set of optimal policies.

\section{Solution Sets for ESR}
\label{sec:esr_solution_sets}
Section \ref{sec:stochastic_dominance_esr} defines a partial ordering over policies under the ESR criterion for unknown utility functions using FSD. In the unknown utility function scenario, it is infeasible to learn a single optimal policy \cite{roijers2013survey}. When a user's utility function is unknown, multi-policy MORL algorithms must be used to learn a set of optimal policies. To apply MORL to the ESR criterion in scenarios with unknown utility, a set of optimal policies under the ESR criterion must be defined. In Section \ref{sec:esr_solution_sets}, FSD is used to define multiple sets of optimal policies for the ESR criterion.

Firstly, a set of optimal policies, known as the undominated set, is defined. The undominated set is defined using FSD, where each policy in the undominated set has an underlying return distribution that is FSD dominant. The undominated set contains at least one optimal policy for all possible monotonically increasing utility functions.

\begin{definition}
\label{definition:undominanted_esr_set}
The undominated set, $U(\Pi)$, is a sub-set of all possible policies for where there exists some utility function, $u$, where a policy's return distribution is FSD dominant.

\[
	U(\Pi) = \left\{\pi \in \Pi\ \middle|\ \exists u, \forall \pi' \in \Pi : \textbf{Z}^{\pi} \succeq_{FSD} \textbf{Z}^{\pi'}\right\}
\]
\end{definition}
However, the undominated set may contain excess policies. For example, under FSD, if two dominant policies have return distributions that are equal, then both policies will be in the undominated set. Given both return distributions are equal, a user with a monotonically increasing utility function will not prefer one policy over the other. In this case, both policies have the same expected utility.
To reduce the number of policies that must be considered at execution time, for each possible utility function we can keep just one corresponding FSD dominant policy; such a set of policies is called a coverage set (CS).
\begin{definition}
\label{definition:undominanted_cs_set}
The coverage set, $CS(\Pi)$, is a subset of the undominated set, $U(\Pi)$, where, for every utility function, $u$, the set contains a policy that has a FSD dominant return distribution,
\[
	CS(\Pi) \subseteq U(\Pi) \wedge \left(\forall u, \exists \pi \in CS(\Pi), \forall{\pi' \in \Pi} : \mathbf{Z}^{\pi} \succeq_{FSD} \mathbf{Z}^{\pi'} \right)
\]
\end{definition}
In practice, a decision maker may aim to learn the smallest possible set of optimal policies. However, FSD considered in this work does not have a strict inequality condition. Moreover, the undominated set generated using FSD may contain excess policies. Therefore, to compute a coverage set in practice where each optimal policy has a unique return distribution, we define expected scalarised returns dominance (ESR dominance). In contrast to FSD, ESR dominance ensures that an explicitly strict inequality exists.
\begin{definition}
\label{definition:esr_dominance}
For random vectors \textbf{X} and \textbf{Y}, \textbf{X} $\succ_{ESR}$ \textbf{Y} for all decision makers with a monotonically increasing utility function if, and only if, the following is true:
\[
\textbf{X} \succ_{ESR} \textbf{Y} \Leftrightarrow
\]
\[
\forall u :  ( \forall {\textbf{v}}: P(u(\textbf{X}) > u({\textbf{v}})) \geq P( u(\textbf{Y}) > u({\textbf{v}})) 
\]
\[
  \wedge \exists \,\textbf{v} : P( u(\textbf{X}) > u({\textbf{v}})) > P( u(\textbf{Y}) > u({\textbf{v}}))).
\]
\end{definition}

ESR dominance (Definition \ref{definition:esr_dominance}) extends FSD, however, ESR dominance is a more strict dominance criterion. For FSD, policies that have equal return distributions are considered dominant policies, which is not the case under ESR dominance. Therefore, if a random vector is ESR dominant, the random vector has a greater expected utility than all ESR dominated random vectors.
Theorem \ref{theorem:esr_dominance} proves that ESR dominance satisfies the ESR criterion when the utility function of the user is unknown for all monotonically increasing utility functions. Theorem \ref{theorem:esr_dominance} focuses on random vectors $\textbf{X}$ and $\textbf{Y}$ where each random vector has two random variables, such that $\textbf{X} = [X_{1}, X_{2}]$ and $\textbf{Y} = [Y_{1}, Y_{2}]$. $F_{\textbf{X}}$ and $F_{\textbf{Y}}$ are the corresponding CDFs and $\textbf{v} = [t, z]$. However, Theorem \ref{theorem:esr_dominance} can easily be extended for random vectors with $n$ random variables ($\textbf{X} = [X_{1}, X_{2}, ..., X_{n}]$).

\begin{theorem}
\label{theorem:esr_dominance}
A random vector, \textbf{X}, is preferred to a random vector, \textbf{Y}, by all decision makers with a monotonically increasing utility function if, and only if, \textbf{X} $\geq_{ESR}$ \textbf{Y}:
\[
\textbf{X} \succ_{ESR} \textbf{Y} \implies \mathbb{E}(u(\textbf{X})) > \mathbb{E}(u(\textbf{Y}))
\]
\end{theorem}
\begin{proof}
$\textbf{X}$ and $\textbf{Y}$ are random vectors with $n$ random variables.
If \textbf{X} $\succ_{ESR}$ \textbf{Y} the following two conditions must be met for all $u$:
\begin{enumerate}
    \item $\forall {\textbf{v}}: P(u(\textbf{X}) > u({\textbf{v}})) \geq P( u(\textbf{Y}) > u({\textbf{v}}))$ 
    \item $\exists \,\textbf{v} : P( u(\textbf{X}) > u({\textbf{v}})) > P( u(\textbf{Y}) > u({\textbf{v}}))$
\end{enumerate}
From Definition \ref{definition:multi_objective_fsd_esr}, if \textbf{X} $\succeq_{FSD}$ \textbf{Y} then the following is true:
\[
\forall u :  \forall {\textbf{v}}: P(u(\textbf{X}) > u({\textbf{v}})) \geq P( u(\textbf{Y}) > u({\textbf{v}})) 
\]
If \textbf{X} $\succeq_{FSD}$ \textbf{Y}, then, from Theorem \ref{theorem:multivariate_FSD}, the following is true:
\[
\mathbb{E}(u(\textbf{X})) \geq \mathbb{E}(u(\textbf{Y}))
\]
If condition $1$ is satisfied, the expected utility of $\textbf{X}$ is at least equal to the expected utility of $\textbf{Y}$, then:
\[
\mathbb{E}(u(\textbf{X})) = \int_{-\infty}^{\infty} \int_{-\infty}^{\infty} u(\textbf{z})f_{\textbf{X}}(t, z) \, dt \, dz
\]
\[
\mathbb{E}(u(\textbf{Y})) = \int_{-\infty}^{\infty} \int_{-\infty}^{\infty} u(\textbf{z})f_{\textbf{Y}}(t, z) \, dt \, dz
\]
In order to satisfy condition 2, some limits must exist to give the following,
\[
\int_{a}^{b}  \int_{c}^{d} u(t, z)f_{\textbf{X}}(t, z) \, dt \, dz > \int_{a}^{b}  \int_{c}^{d} u(t, z) f_{\textbf{Y}}(t, z) \, dt \, dz
\]
The minimum requirement to satisfy condition 1 is:
\[
\int_{-\infty}^{\infty}  \int_{-\infty}^{\infty} u(t, z)f_{\textbf{X}}(t, z) \, dt \, dz = \int_{-\infty}^{\infty}  \int_{-\infty}^{\infty} u(t, z)f_{\textbf{Y}}(t, z) \, dt \, dz
\]
If condition 1 is satisfied, to satisfy condition 2 some limits must exist:
\[
\int_{a}^{b}  \int_{c}^{d} u(t, z)f_{\textbf{X}}(t, z) \, dt \, dz > \int_{a}^{b}  \int_{c}^{d} u(t, z)f_{\textbf{Y}}(t, z) \, dt \, dz.
\]
Therefore,
\[
\int_{-\infty}^{a}  \int_{-\infty}^{c} u(t, z)f_{\textbf{X}}(t, z) \, dt \, dz \, + \int_{a}^{b}  \int_{c}^{d} u(t, z)f_{\textbf{X}}(t, z) \, dt \,dz \, +
\]
\[
\int_{b}^{\infty}  \int_{d}^{\infty} u(t, z)f_{\textbf{X}}(t, z) \,dt \, dz > \int_{-\infty}^{a}  \int_{-\infty}^{c} u(t, z)f_{\textbf{Y}}(t, z) \, dt \, dz \, +
\]
\[
\int_{a}^{b}  \int_{c}^{d} u(t, z)f_{\textbf{Y}}(t, z) \, dt \, dz \, + \int_{b}^{\infty}  \int_{d}^{\infty} u(t, z)f_{\textbf{Y}}(t, z) \, dt \, dz.
\]
Finally,
\[
\int_{-\infty}^{\infty}  \int_{-\infty}^{\infty} u(t, z)f_{\textbf{X}}(t, z) \, dt \, dz > \int_{-\infty}^{\infty}  \int_{-\infty}^{\infty} u(t, z)f_{\textbf{Y}}(t, z) \, dt \, dz
\]
if \textbf{X} $\succ_{ESR}$ \textbf{Y}, then,
\[
\mathbb{E}(u(\textbf{X})) > \mathbb{E}(u(\textbf{Y})).
\]
\end{proof}
In the ESR dominance criterion defined in Definition \ref{definition:esr_dominance}, the utility of different vectors is compared. However, it is not possible to calculate the utility of a vector when the utility function is unknown. In this case, Pareto dominance \cite{pareto1896dominance} can be used instead to determine the relative utility of the vectors being compared.
\begin{definition} $\textbf{A}$ Pareto dominates ($\succ_{p}$) $\textbf{B}$ if the following is true:
\label{defition:strict_pareto_dominace}
\begin{equation}
    \mathbf{A} \succ_{p} \mathbf{B} \Leftrightarrow (\forall i: \mathbf{A}_{i} \ge \mathbf{B}_{i}) \land (\exists i: \mathbf{A}_i > \mathbf{B}_i).
\end{equation}
\end{definition}
For monotonically increasing utility functions, if the value of an element of the vector increases, then the scalar utility of the vector also increases. Therefore, using Definition \ref{defition:strict_pareto_dominace}, if vector $\mathbf{A}$ Pareto dominates vector $\mathbf{B}$, for a monotonically increasing utility function, $\mathbf{A}$ has a higher utility than $\mathbf{B}$. To make ESR comparisons between return distributions, Pareto dominance can be used.
\begin{definition}
\label{def:esr_pdf}
For random vectors \textbf{X} and \textbf{Y}, \textbf{X} $\succ_{ESR}$ \textbf{Y} for all monotonically increasing utility functions if, and only if, the following is true:
\[
\textbf{X} \succ_{ESR} \textbf{Y} \Leftrightarrow
\]
\[
  \forall {\bf v}: P( {\bf X} >_P {\bf v}) \geq P( {\bf Y} >_P {\bf v})
   \wedge \exists {\bf v} : P( {\bf X} >_P {\bf v}) > P( {\bf Y} >_P {\bf v}).
\]
\end{definition}
It is also possible to calculate ESR dominance by comparing the CDFs of random vectors. Using the CDF guarantees a higher expected utility. Using the CDF we compare the cumulative probabilities for a given vector, where a lower cumulative probability is preferred. ESR dominance with the CDF does not require any information about the utility function of a user and therefore can be used in the unknown utility function scenario.

\begin{definition}
\label{def:esr_cdf}
For random vectors \textbf{X} and \textbf{Y}, \textbf{X} $\succ_{ESR}$ \textbf{Y} for all monotonically increasing utility functions if, and only if, the following is true:

\[
\textbf{X} \succ_{ESR} \textbf{Y} \Leftrightarrow
\]
\[
  \forall {\bf v}: F_{\textbf{X}}({\bf v}) \leq F_{\textbf{Y}}({\bf v})
   \wedge \exists {\bf v} : F_{\textbf{X}}({\bf v}) < F_{\textbf{Y}}({\bf v}).
\]
\end{definition}
Therefore, we can use either Definition \ref{def:esr_pdf} or Definition \ref{def:esr_cdf} to calculate ESR dominance to give a partial ordering over policies.
 
\begin{definition}
\label{defn:esr_policies}
For return distributions $\textbf{Z}^{\pi}$ and $\textbf{Z}^{\pi'}$ for policies $\pi$ and $\pi'$, $\pi$ is preferred over $\pi'$ by all decision makers with a monotonically increasing utility function if, and only if, the following is true:
\[
\textbf{Z}^{\pi} \succ_{ESR} \textbf{Z}^{\pi'} 
\]
\end{definition}
Using ESR dominance, it is possible to define a set of optimal policies, known as the \emph{ESR set}.
\begin{definition}
The \emph{ESR set}, $ESR(\Pi)$, is a sub-set of all policies where each policy in the \emph{ESR set} is ESR dominant,
\begin{equation}
    ESR(\Pi) = \{ 
    \pi \in \Pi\  |\ \nexists \pi'\in\Pi : \mathbf{Z}^{\pi'} \succ_{ESR} \mathbf{Z}^\pi
    \}.
\end{equation}
\end{definition}
The \emph{ESR set} is a set of non-dominated policies, where each policy in the \emph{ESR set} is ESR dominant. The \emph{ESR set} can be considered a coverage set, given no excess policies exist in the \emph{ESR set}. It is viable for a multi-policy MORL method to use ESR dominance to construct the \emph{ESR set}.

\section{Multi-Objective Tabular Distributional Reinforcement Learning}
\label{sec:modrl}
Traditionally in the MORL literature, multi-objective methods learn a set of optimal solutions when the utility function of a user is unknown or hard to specify \cite{roijers2013survey,hayes2021practical}. The current MORL literature focuses only on methods which learn the optimal set of policies under the SER criterion \cite{vanmoffaert2014multiobjective,wang2012multi}. As already highlighted, the ESR criterion has largely been ignored by the MORL community, with a few exceptions \cite{roijers2018multi,hayes2021dmcts,vamplew2021}. In Section \ref{sec:modrl} we address this research gap and we present a novel multi-objective tabular distributional reinforcement learning (MOTDRL) algorithm that learns an optimal set of policies for the ESR criterion, also known as the \emph{ESR set}, for multi-objective multi-armed bandit (MOMAB) problems.

MOTDRL learns the return distribution for a policy by sampling each available arm in a MOMAB setting and maintains a multivariate distribution over the returns received. Given MOTDRL only considers MOMAB problem domains, MOTDRL maintains a distribution per arm and updates the distribution after each timestep with the return vector received from executing the sampled arm. When optimising under the ESR criterion it is critical that a MORL method learns the underlying distribution over the returns. Other distributional MORL methods, such as bootstrap Thompson sampling \cite{hayes2021dmcts}, cannot be used to learn a set of optimal policies under the ESR criterion when the utility function is unknown. Such methods learn a distribution over the mean returns. In scenarios where the utility function is unknown or unavailable, such methods would invalidate the ESR criterion as a distribution over mean return vectors would be computed. Given a distribution must be used when learning the \emph{ESR set}, new distributional MORL methods must be formulated to learn the underlying return distributions.

MOTDRL can learn the underlying return distribution for an arm by maintaining a tabular representation of the underlying multivariate distribution. To maintain a tabular representation of a multivariate distribution we initialise a $Z$-table for each arm where the $Z$-table has an axis per objective. The $Z$-table maintains a count of the number of times a return vector is received for a given arm. The size of each $Z$-table is initialised using the parameters $\textbf{R}_{min}$ and $\textbf{R}_{max}$ which are the minimum and maximum returns obtainable for any of the objectives in the given environment. Therefore, each axis in the $Z$-table will use $\textbf{R}_{min}$ and $\textbf{R}_{max}$ to define the length of the axis, where each index value of the $Z$-table is initialised to $0$. Using $\textbf{R}_{min}$ and $\textbf{R}_{max}$ as initialisation parameters, a $Z$-table can be constructed which contains an index for all possible return vectors in a given problem domain. Figure \ref{table:z_table} visualises an initialised $Z$-table for a multi-objective problem with two objectives where $\textbf{R}_{min}$ = 0 and $\textbf{R}_{max}$ = 5.   

\begin{figure}[h]
    \centering
    \begin{tabular}{| c | c | c | c | c | c | c |}
    \hline
        Z & $x_{2} = 0$ & $x_{2} = 1$ & $x_{2} = 2$ & $x_{2} = 3$ & $x_{2} = 4$ & $x_{2} = 5$ \\
    \hline
         $x_{1} = 0$ & 0 & 0 & 0 & 0 & 0 & 0 \\
     \hline
     $x_{1} = 1$ & 0 & 0 & 0 & 0 & 0 & 0 \\
     \hline
        $x_{1} =2$ & 0 & 0 & 0 & 0 & 0 & 0 \\
        \hline
        $x_{1} =3$ & 0 & 0 & 0 & 0 & 0 & 0 \\
        \hline
        $x_{1} =4$ & 0 & 0 & 0 & 0 & 0 & 0 \\
        \hline
        $x_{1} =5$ & 0 & 0 & 0 & 0 & 0 & 0 \\
    \hline
    \end{tabular}
    \caption{An illustration of an initialised $Z$-table for a problem domain with two objectives, $x_{1}$ and $x_{2}$, with each index value set to 0.}
    \label{table:z_table}
\end{figure}

\begin{algorithm}[h]
\textbf{Input} - arm, $i$ \\
\textbf{Require} - $Z$-table for arm, $i$, $Z_{i}$ \\
Pull arm, $i$, and observe return, $\textbf{R}$. \\
$Z_{i}($\textbf{R}$)$ = $Z_{i}($\textbf{R}$)$ + 1 \\
$N_{i}$ = $N_{i}$ + 1 \\ 
\textbf{return} $Z$-table, $Z_{i}$.
\caption{$Z$-table Update}
\label{alg:distribution_update}
\end{algorithm}

Each $Z$-table can be used to calculate the return distribution of an arm, which can be considered as a policy $\pi$, $\textbf{Z}^{\pi}$ (see Section \ref{sec:esr}). At each timestep, $t$, the returns, $\textbf{R}$, received from pulling arm, $i$, are used to update the $Z$-table. The $Z$-table is used to maintain a count of the number of times the return, $\textbf{R}$, is received. In MOMAB problem domains, the returns received from the execution of an arm represent the full returns of the execution of a policy. To update the $Z$-table, the value at the index corresponding to the return $\textbf{R}$ is incremented by one. To correctly calculate the probability of receiving return $\textbf{R}$ when pulling arm $i$, a counter, $N_{i}$, which represents the number of times arm $i$ has been pulled, must be maintained. Each time arm $i$ is pulled, the counter $N_{i}$ is incremented by one. Algorithm \ref{alg:distribution_update} outlines how the $Z$-table for each arm is updated.

MOTDRL is a multi-policy algorithm that can learn the ESR set using ESR dominance. Using ESR dominance a partial ordering over policies can be determined when the utility function of a user is unknown. Algorithm \ref{alg:motdrl} outlines how MOTDRL learns the ESR set when the utility function of a user in unknown in a MOMAB problem domain. 
In Algorithm \ref{alg:motdrl} $\mathcal{A}$ is defined as a set of available arms, the ESR set is defined as $E$, $D$ is the number of objectives, $n$ is the total number of pulls across all arms, $N_{l}$ and $N_{i}$ are the number of pulls of arms $j$ and $i$, and $|E^{*}|$ is the cardinality of the ESR set, which is known a priori. When learning, the MOTDRL algorithm has priori knowledge of $\mathcal{A}$, $\textbf{R}_{max}$ and $\textbf{R}_{min}$. The agent must has knowledge of $\textbf{R}_{max}$ and $\textbf{R}_{min}$ so the Z-table can be correctly initialised and the agent must know the number of arms in $\mathcal{A}$ for action selection.

\begin{algorithm}
Pull each arm $i$ in $\mathcal{A}$, $\beta$ times \\
Z-table Update(i) $\forall$ i $\in$ $\mathcal{A}$ \\
\Repeat{stopping condition is met}{Find $E$ such that $\forall$ i $\in$ $E$, $\forall$ $j$ \\ $\textbf{Z}^{j} + \sqrt{\frac{2ln(n\sqrt[4]{D|E^{*}|})}{N_{j}}} \nsucc_{ESR} \textbf{Z}^{i}  + \sqrt{\frac{2ln(n\sqrt[4]{D|E^{*}|})}{N_{i}}}$  \label{alg:UpdateDistribution:line:esr}\\Select arm a uniform randomly from $E$\\Z-table Update(i)}
\caption{Multi-Objective Tabular Distributional Reinforcement Learning}
\label{alg:motdrl}
\end{algorithm}

On initialisation each arm is pulled $\beta$ times. The hyperparameter $\beta$ is selected to ensure each arm is pulled sufficiently to build an initial distribution. For optimal performance $\beta$ is set to greater than $1$. For $\beta$ greater than $1$, MOTDRL can build a sufficient initial distribution and can then efficiently explore each arm with the UCB1 statistic. At each timestep, the return distribution of the policies associated with the execution of each arm is calculated. The ESR set, $E$, is then calculated from the resulting return distributions. Therefore, for  all  the  non-optimal  arms $l \not \in E$, there exists an ESR dominant arm $i \in E$ that ESR dominates the arm $l$.

To calculate ESR dominance required in Algorithm \ref{alg:motdrl} at Line \ref{alg:UpdateDistribution:line:esr}, it is critical to compute both the PDF and CDF of the underling return distribution of a policy. The PDF can be calculated by computing the probability of receiving individual returns. Combining the $Z$-table and $N$ for an arm, $i$, it is possible to compute the probability of receiving each return in a given problem domain, since the following is true:
\begin{equation}
    f_{\textbf{X}}(x_{1}, x_{2}, ..., x_{n}) = P(\textbf{X} = x_{1}, \textbf{X} = x_{2}, ..., \textbf{X} = x_{n}) = \frac{Z_{i}(x_{1}, x_{2}, ..., x_{n})}{N_{i}}
    \label{eqn:pdf_z_table}
\end{equation}
Once the PDF has been computed using Equation \ref{eqn:pdf_z_table}, it is possible to compute the CDF. Since the following is true:
\begin{equation}
\begin{split}
    F_{\textbf{X}}(x_{1}, x_{2}, ..., x_{n}) = & P(\textbf{X} \leq x_{1}, \textbf{X} \leq x_{2}, ..., \textbf{X} \leq x_{n}) \\ = &\sum_{x_{a} \leq x_{1}} \sum_{x_{b} \leq x_{2}}... \sum_{x_{k} \leq x_{n}} P(\textbf{X} = x_{a}, \textbf{X} = x_{b}, ..., \textbf{X} = x_{k}) \\ = &\sum_{x_{a} \leq x_{1}} \sum_{x_{b} \leq x_{2}}... \sum_{x_{k} \leq x_{n}} \frac{Z_{i}(x_{a}, x_{b}, ..., x_{k})}{N_{i}}
\end{split}
\label{eqn:cdf_z_table}
\end{equation}

Using the PDF and the CDF of a return distribution it is possible to calculate ESR dominance using Definition \ref{def:esr_pdf} or Definition \ref{def:esr_cdf}. Both methods can be used to calculate ESR dominance.

To efficiently explore all available arms, MOTDRL uses the UCB1 statistic presented by Drugan et al. \cite{drugan2013momab}. MOTDRL uses UCB1 to transform the PDF of the underlying return distribution. MOTDRL tranforms the PDF by adding the UCB1 statistic, computed at Line \ref{alg:UpdateDistribution:line:esr} in Algorithm \ref{alg:motdrl}, to the PDF. By summing the UCB1 statistic and the PDF, the PDF is shifted relative to the value of the computed UCB1 statistic. The CDF can then calculated based on the transformed PDF and ESR dominance can then be computed. 

Transforming the PDF using the UCB1 statistic ensures that there is sufficient exploration of all available arms during experimentation. However, as the number of pulls of a given arm increases the UCB1 statistic decreases, which decreases exploration. Over time the UCB1 statistic's effect on the PDF and CDF becomes negligible. At such a point, MOTDRL can exploit the return distributions learned during exploration and compute the ESR set.

Given MOTDRL is a multi-policy algorithm, MOTDRL can be used in the unknown utility function scenario (see Figure \ref{fig:unknown_utility_function_scenario}). During the learning phase MOTDRL learns the ESR set by utilising the steps in Algorithm \ref{alg:motdrl}. In Section \ref{sec:experiments} we deploy MOTDRL in two multi-objective multi-armed bandit settings to show MOTDRL can learn the ESR set. It is important to note that the experiments presented only consider the learning phase.
\section{Experiments}
\label{sec:experiments}
In order to evaluate the MOTDRL algorithm, we evaluate MOTDRL in multiple settings \footnote{It is important to note for each experiment the results of the learning phase is presented where and a set of optimal policies is computed. The selection phase and execution phase are not included in the evaluation of MOTDRL.}. Before experimentation we define a metric that can be used to evaluate the performance of multi-policy ESR methods. We then evaluate MOTDRL in a multi-objective multi-armed bandit setting. Finally, we define a new multi-objective multi-armed bandit problem domain known as the Vaccine Recommender System (VRS) environment and evaluate MOTDRL using the VRS environment.

\subsection{Evaluation Metrics}
The standard metrics for MORL \cite{vamplew2011evaluation_methods,zintgraf2015quality_ser,yang2019} are not suitable to evaluate a multi-policy method under the ESR criterion since they are designed to specifically evaluate SER methods. To evaluate MORL algorithms under the ESR criterion, we adapt the coverage ratio metric used by Yang et al. \cite{yang2019} for the ESR criterion. The coverage ratio evaluates the agent’s ability to recover optimal solutions in the ESR set ($E$). If $\mathcal{F} \subseteq {R}^m$ is the set of solutions found by the agent, we define the following:
\begin{equation}
\mathcal{F} \cap_{\epsilon} E := \{ Z^{\pi} \in \mathcal{F} \, | \, \exists Z^{\pi'} \in E \, s.t \, \sup\limits_{\textbf{x}} |F_{Z^{\pi}}(\textbf{x}) - F_{Z^{\pi'}}(\textbf{x}) \, | \leq \epsilon \},
\label{eqn:coverage_ratio}
\end{equation}
where $\textbf{x} = [x_{1}, x_{2}, ..., x_{D}]$ and $D$ is equal to the number of objectives. 
Equation \ref{eqn:coverage_ratio} uses the Kolmogorov–Smirnov statistic \cite{darling1957kstest} (Equation \ref{eqn:ks_statistic}), where $ \sup\limits_{\textbf{x}}$ is the supremum of the set of distances. The Kolmogorov–Smirnov statistic takes the largest absolute difference between the two CDFs across all $\textbf{x}$ values,
\begin{equation}
\sup\limits_{\textbf{x}} |F_{Z^{\pi}}(\textbf{x}) - F_{Z^{\pi'}}(\textbf{x})|.
\label{eqn:ks_statistic}
\end{equation}
The Kolmogorov–Smirnov statistic returns a minimum value of $0$ and a maximum value of $1$. If two CDFs are equal then the Kolmogorov–Smirnov statistic will return a value of $0$.

The coverage ratio is then defined as:
\begin{equation}
    F_{1} = 2 \cdot \frac{precision \cdot recall}{precision + recall},
\end{equation}
where precision $= | \mathcal{F} \cap_{\epsilon} E| / |\mathcal{F}|$ indicating the fraction of optimal solutions among the retrieved solutions, and the recall $= | \mathcal{F} \cap_{\epsilon} E| / |E|$ indicating the fraction of optimal instances that have been retrieved over the total amount of optimal solutions \cite{yang2019}.

\subsection{Multi-Objective Multi-Armed Bandit Environment}
\label{sec:experiments_momab}
In Section \ref{sec:experiments_momab} we evaluate MOTDRL in a MOMAB setting. To evaluate MOTDRL, we consider a bi-objective MOMAB with five arms. Table \ref{table:momab_distributions} outlines the number of possible outcomes obtainable when selecting a given arm and the corresponding probabilities. Table \ref{table:momab_distributions} is unknown to the agent, and the agent aims to learn each distribution per arm and prune the ESR dominated arms from consideration. In the MOMAB setting the ESR set is known a priori where the return distributions for $arm_1$ and $arm_5$ are ESR dominant and therefore the ESR set only contains $arm_1$ and $arm_5$. 

To evaluate MOTDRL in a MOMAB environment we set $R_{min} = 0$, $R_{max} = 10$, $D$ = 2, $\beta$ = 5 and $|E^{*}|$ = 2. To compute the coverage ratio we set $\epsilon = 0.01$. All experiments in this setting are averaged over 10 runs. 
\begin{table}[H]

    \centering
    \begin{tabular}{| c | c |}
    \multicolumn{2}{c}{$arm_1$} \\
    \hline
         P(Arm 1 = $\textbf{R}$) & \textbf{R} \\
    \hline
         0.4 & (0, 1) \\
    \hline 
          0.6 & (5, 4) \\
    \hline
    \end{tabular}
    \begin{tabular}{| c | c |}
    \multicolumn{2}{c}{$arm_2$} \\
    \hline
         P(Arm 2 = $\textbf{R}$) & \textbf{R} \\
    \hline
         0.85 & (1, 0) \\
    \hline 
          0.15 & (3, 2) \\

    \hline
    \end{tabular}
    \begin{tabular}{| c | c |}
    \multicolumn{2}{c}{$arm_3$} \\
    \hline
         P(Arm 3= $\textbf{R}$) & \textbf{R} \\
    \hline
         0.75 & (2, 0) \\
    \hline 
          0.25 & (4, 2) \\
    \hline
    \end{tabular}
    \begin{tabular}{| c | c |}
    \multicolumn{2}{c}{$arm_4$} \\
    \hline
         P(Arm 4 = $\textbf{R}$) & \textbf{R} \\
    \hline
         0.8 & (0, 1) \\
    \hline 
          0.2 & (1, 2) \\
    \hline
    \end{tabular}
    \begin{tabular}{| c | c |}
    \multicolumn{2}{c}{$arm_5$} \\
    \hline
         P(Arm 5 = $\textbf{R}$) & \textbf{R} \\
    \hline 
          0.7 & (2, 0) \\
    \hline
         0.3 & (4, 5) \\
    \hline
    \end{tabular}
    \caption{A MOMAB with 5 arms where selecting an arm has two outcomes and two objectives.}
    \label{table:momab_distributions}
\end{table}

\begin{figure}[h]
    \centering
    \includegraphics[height = 7cm, width=9cm]{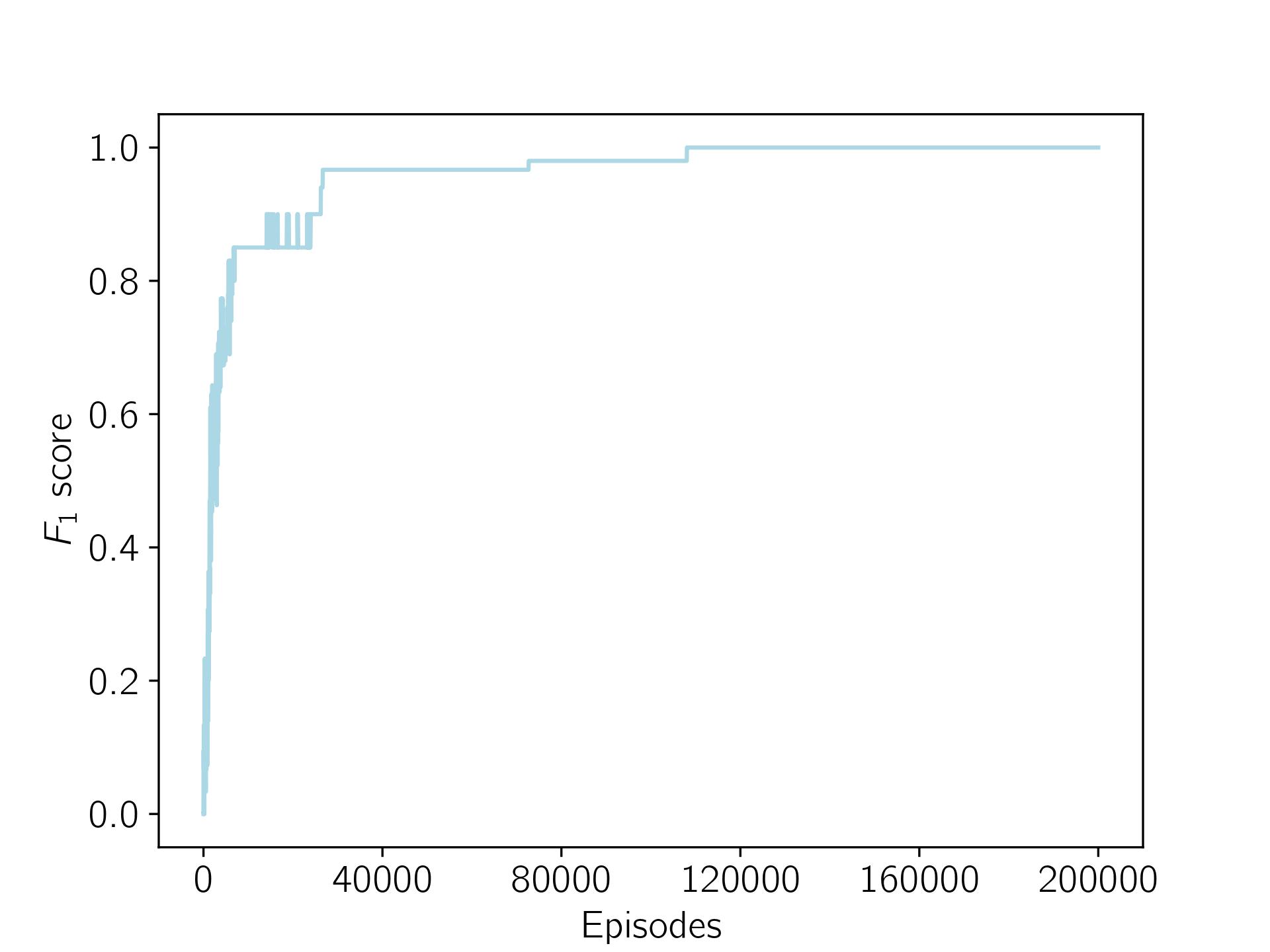}
    \caption{Results from the MOMAB environment. MOTDRL is able to learn the ESR set as MOTDRL converges to the optimal coverage ratio since the $F_1$ score reaches the maximum possible value of 1.} 
    \label{fig:momab_f1}
\end{figure}
MOTDRL is able to learn the underlying return distributions for each arm in the MOMAB setting. Using the return distributions for each arm, MOTDRL can learn the ESR set in the MOMAB environment. In Figure \ref{fig:momab_f1}, we plot the coverage ratio as the $F_{1}$ score. MOTDRL converges to the optimal $F_1$ score of $1$. MOTDRL converges to the optimal $F_1$ score after $100,000$ episodes. An optimal $F_1$ score can only be achieved when all policies in the ESR set have been learned by the agent.

MOTDRL computes the ESR set for the MOMAB environment during the learning phase. The learned ESR set contains two arms; $arm_1$ and $arm_5$. Both $arm_1$ and $arm_5$ are ESR dominant therefore any user with a monotonically increasing utility function would prefer $arm_1$ or $arm_5$ over all other available arms in the MOMAB problem. MOTDRL will return the ESR set to the user during the selection phase. In practice, a user will select a policy form the ESR set which best reflects their preferences and the selected policy will be executed.

\begin{figure}
    \centering
    \begin{subfigure}{.5\textwidth}
         \includegraphics[width=6.25cm]{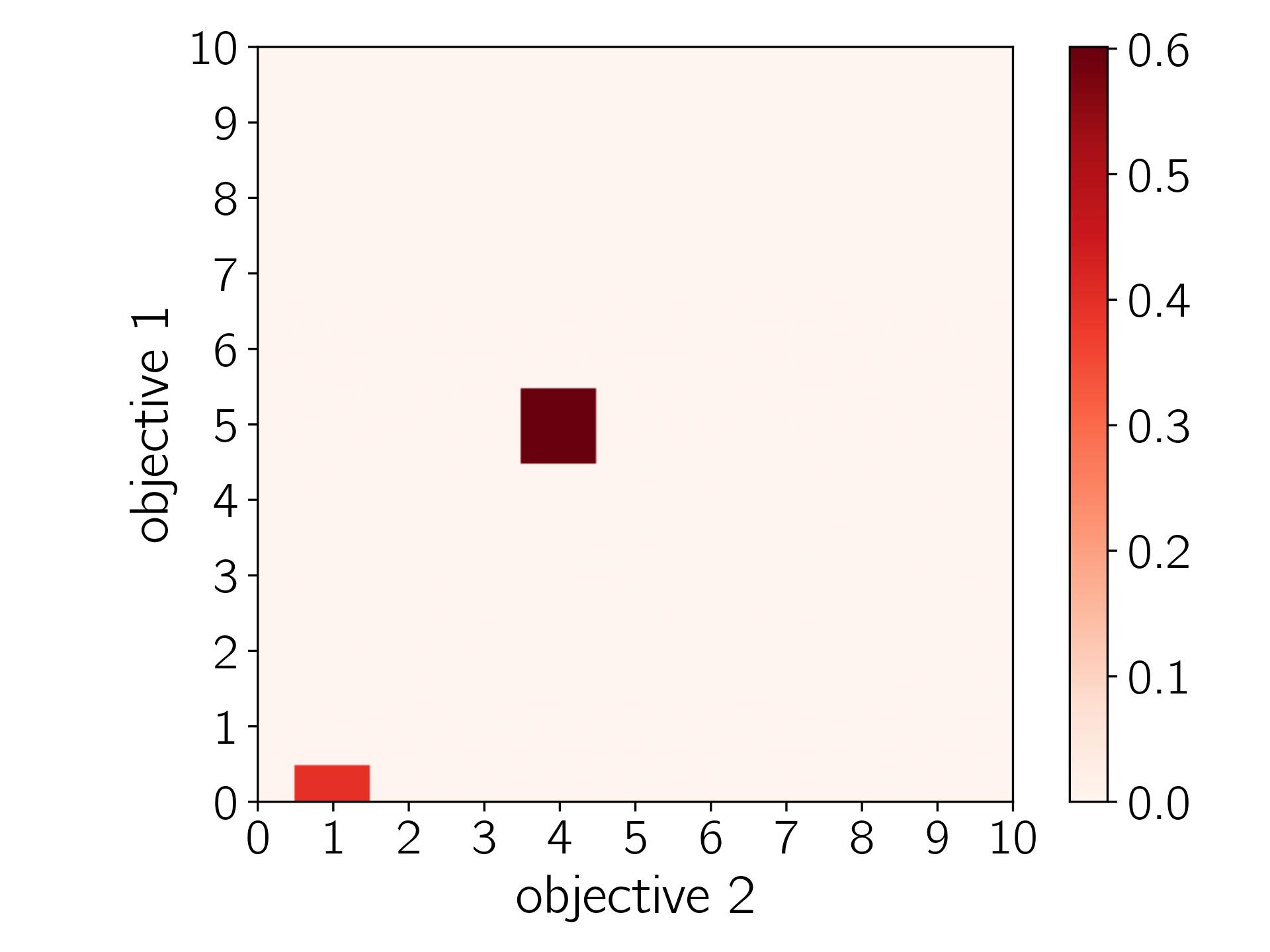}
    \end{subfigure}%
    \begin{subfigure}{.5\textwidth}
    \includegraphics[width=6.25cm]{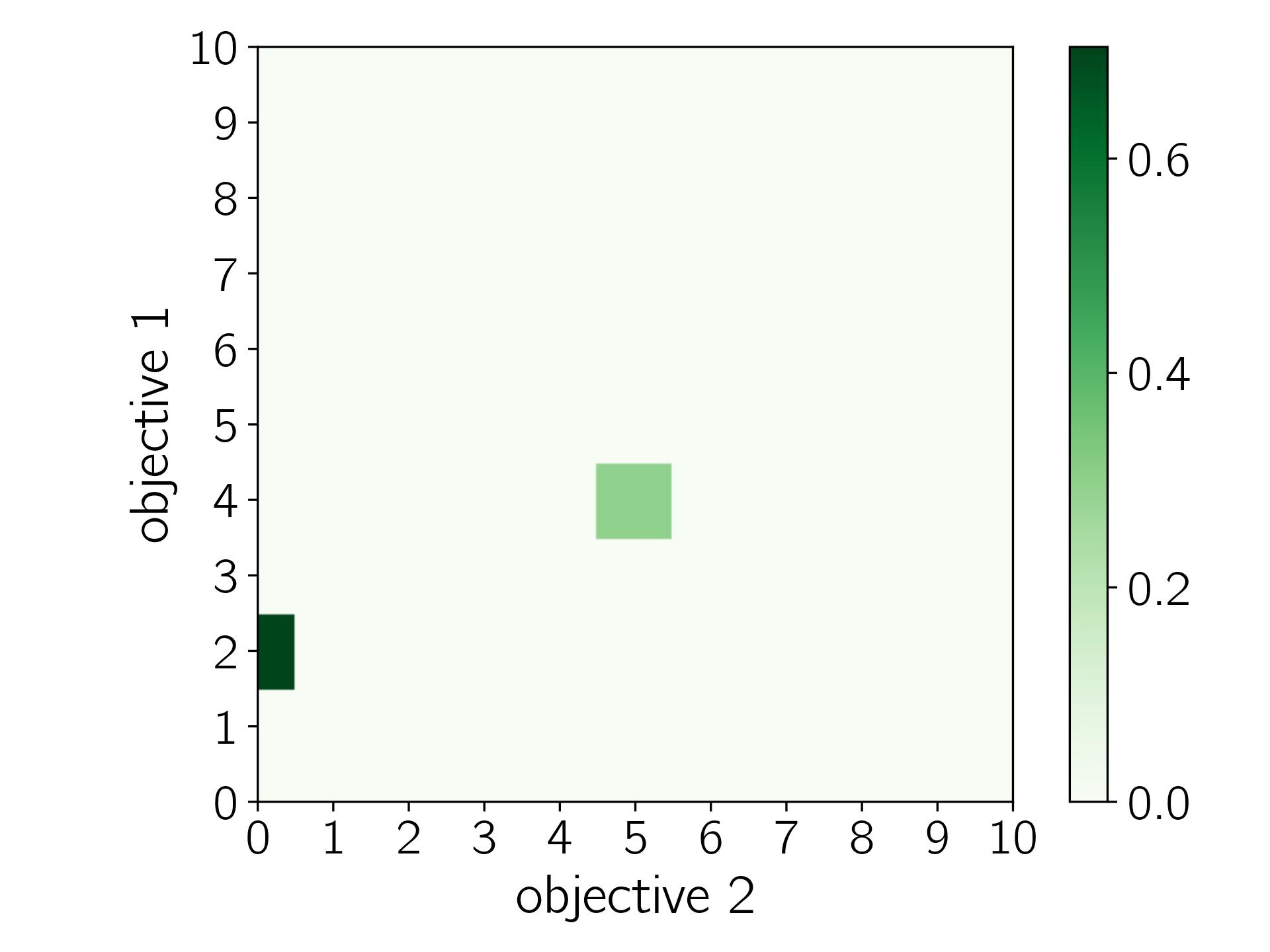}
    \end{subfigure}
    \caption{Heatmaps for each return distribution in the ESR set learned by MOTDRL in the MOMAB environment. The left heatmap describes the return distribution for $arm_1$ learned by MOTDRL and the right heatmap describes the return distribution for $arm_5$ learned by MOTDRL.} 
    \label{fig:momab_heatmaps}
\end{figure}
Given ESR dominance is a new solution concept, we utilise Figure \ref{fig:momab_heatmaps}, Figure \ref{fig:momab_cdfs} and Figure \ref{fig:multi_plots_momab} to give the reader some intuition about ESR dominance. Figure \ref{fig:momab_heatmaps} displays the return distributions in the ESR set learned by MOTDRL as heatmaps. Each heatmap in Figure \ref{fig:momab_heatmaps} corresponds to the probabilities highlighted for $arm_1$ (left) and $arm_5$ (right) in Table \ref{table:momab_distributions}. 

\begin{figure}
    \centering
    \begin{subfigure}{.5\textwidth}
         \includegraphics[width=6cm]{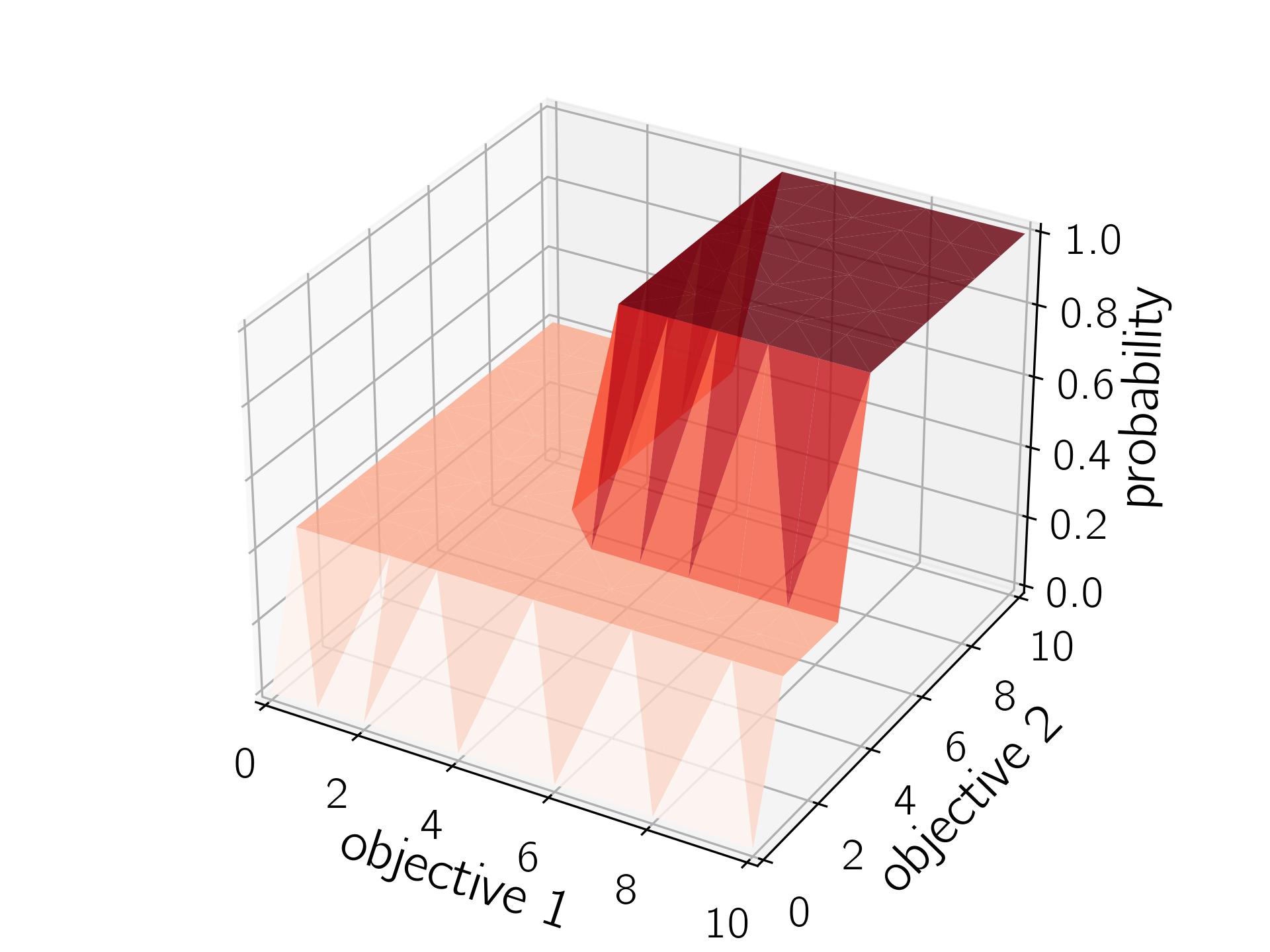}
    \end{subfigure}%
    \begin{subfigure}{.5\textwidth}
    \includegraphics[width=6cm]{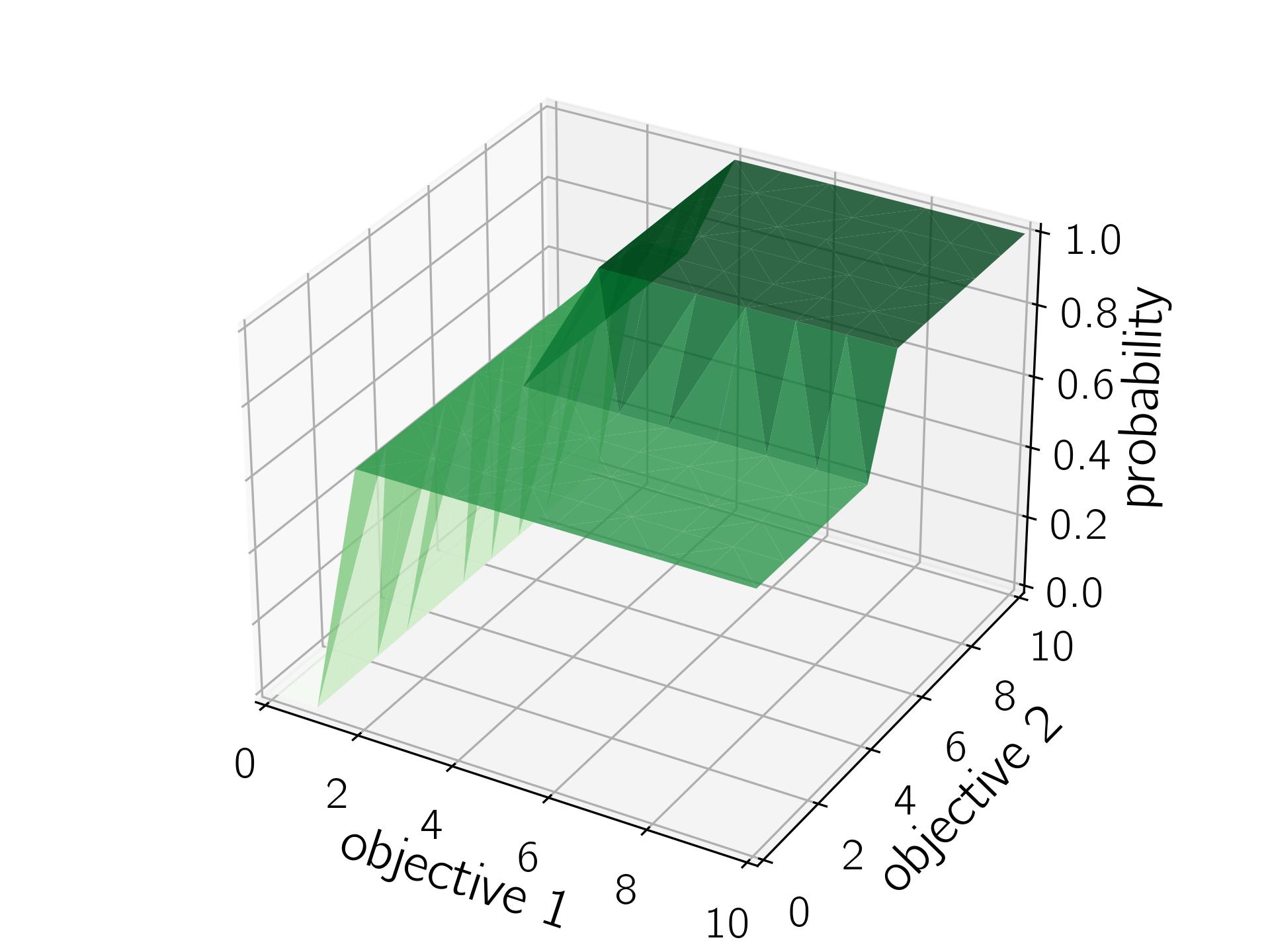}
    \end{subfigure}
    \caption{CDFs for each policy in the ESR set learned by MOTDRL in the MOMAB environment. The left figure describes the CDF for $arm_1$ learned by MOTDRL and the right figure describes the CDF for $arm_5$ learned by MOTDRL.} 
    \label{fig:momab_cdfs}
\end{figure}

Figure \ref{fig:momab_cdfs} displays the CDFs for each return distribution in the ESR set learned by MOTDRL. The CDF is used to calculate ESR dominance and the CDFs in Figure \ref{fig:momab_cdfs} correspond to the CDFs of $arm_1$ (left) and $arm_5$ (right) in Table \ref{table:momab_distributions}.

\begin{figure}
     \centering
    \includegraphics[width=6.25cm]{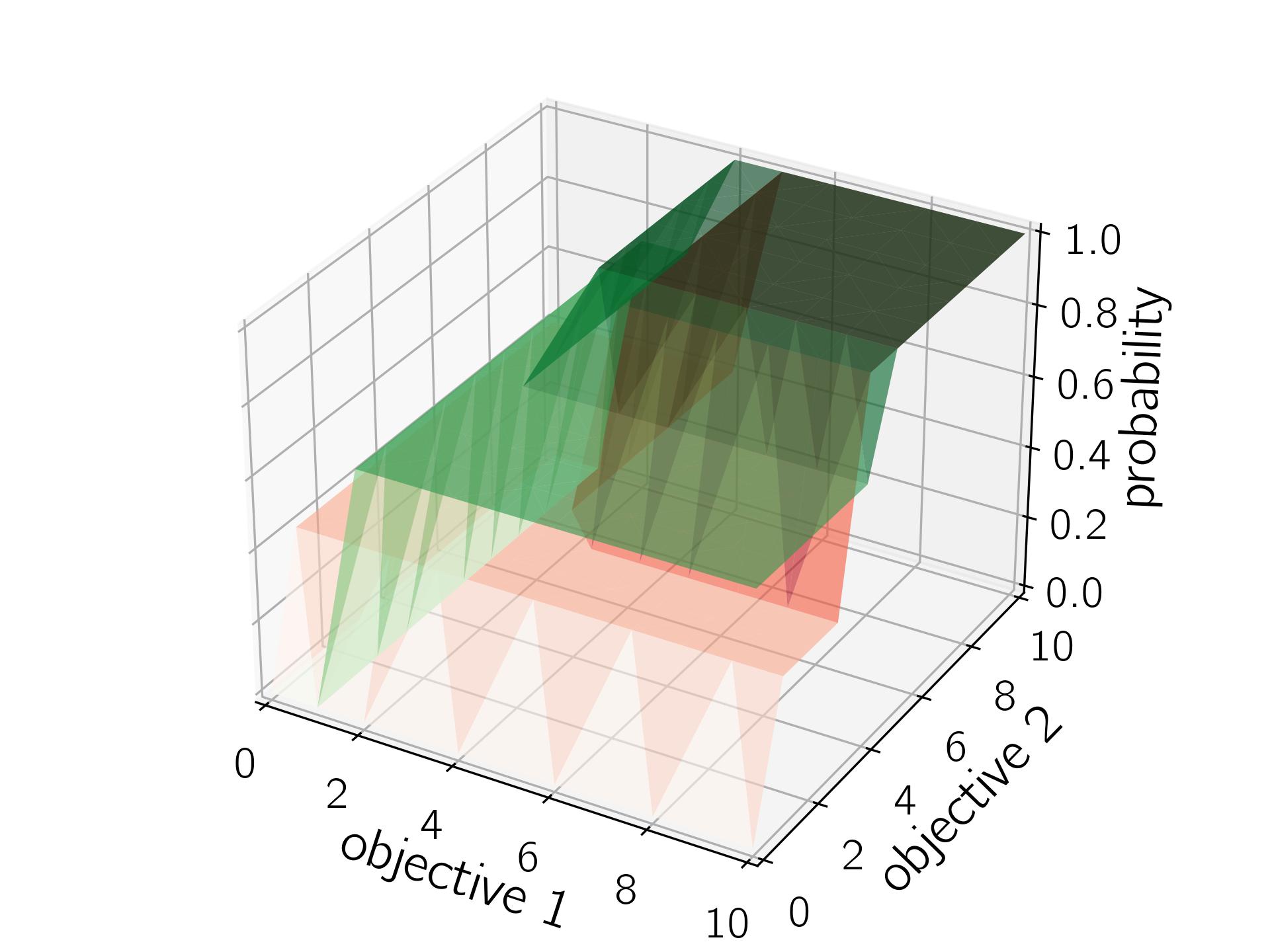}
    \caption{The CDFs for $arm_1$ and $arm_5$ intersect at multiple points. Therefore, as per Definition \ref{definition:esr_dominance}: $arm_1$ $\nsucc_{ESR}$ $arm_5$ and $arm_5$ $\nsucc_{ESR}$ $arm_1$.}
    \label{fig:multi_plots_momab}
\end{figure}

Figure \ref{fig:multi_plots_momab} describes how $arm_1$ $\nsucc_{ESR}$ $arm_5$ and $arm_5$ $\nsucc_{ESR}$ $arm_1$ given the CDFs for $arm_1$ and $arm_5$ intersect at multiple points (see Definition \ref{definition:esr_dominance}).

\begin{figure}
    \centering
    \begin{subfigure}{.5\textwidth}
         \includegraphics[width=6.05cm]{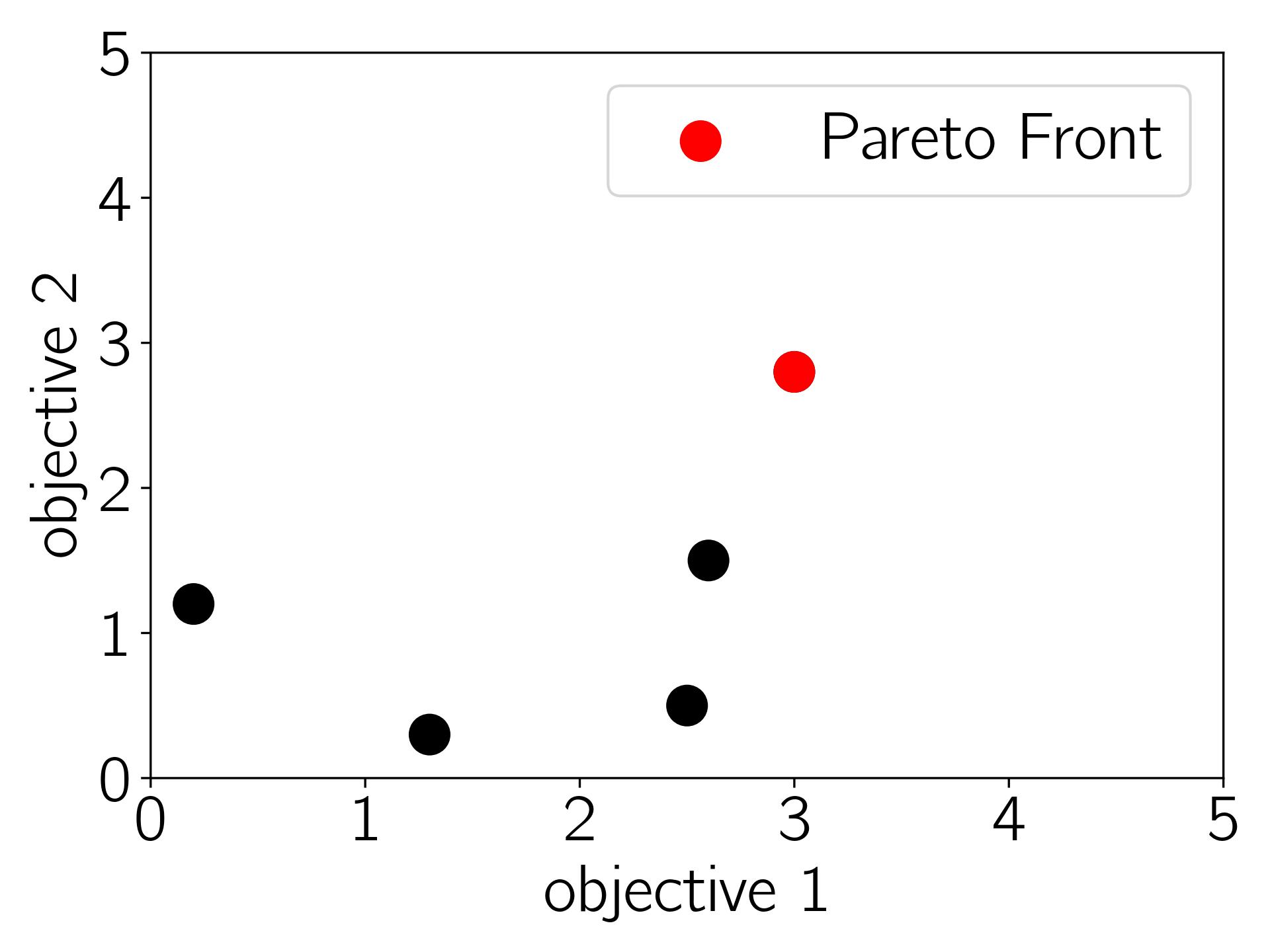}
    \end{subfigure}%
    \begin{subfigure}{.5\textwidth}
    \includegraphics[width=6.05cm]{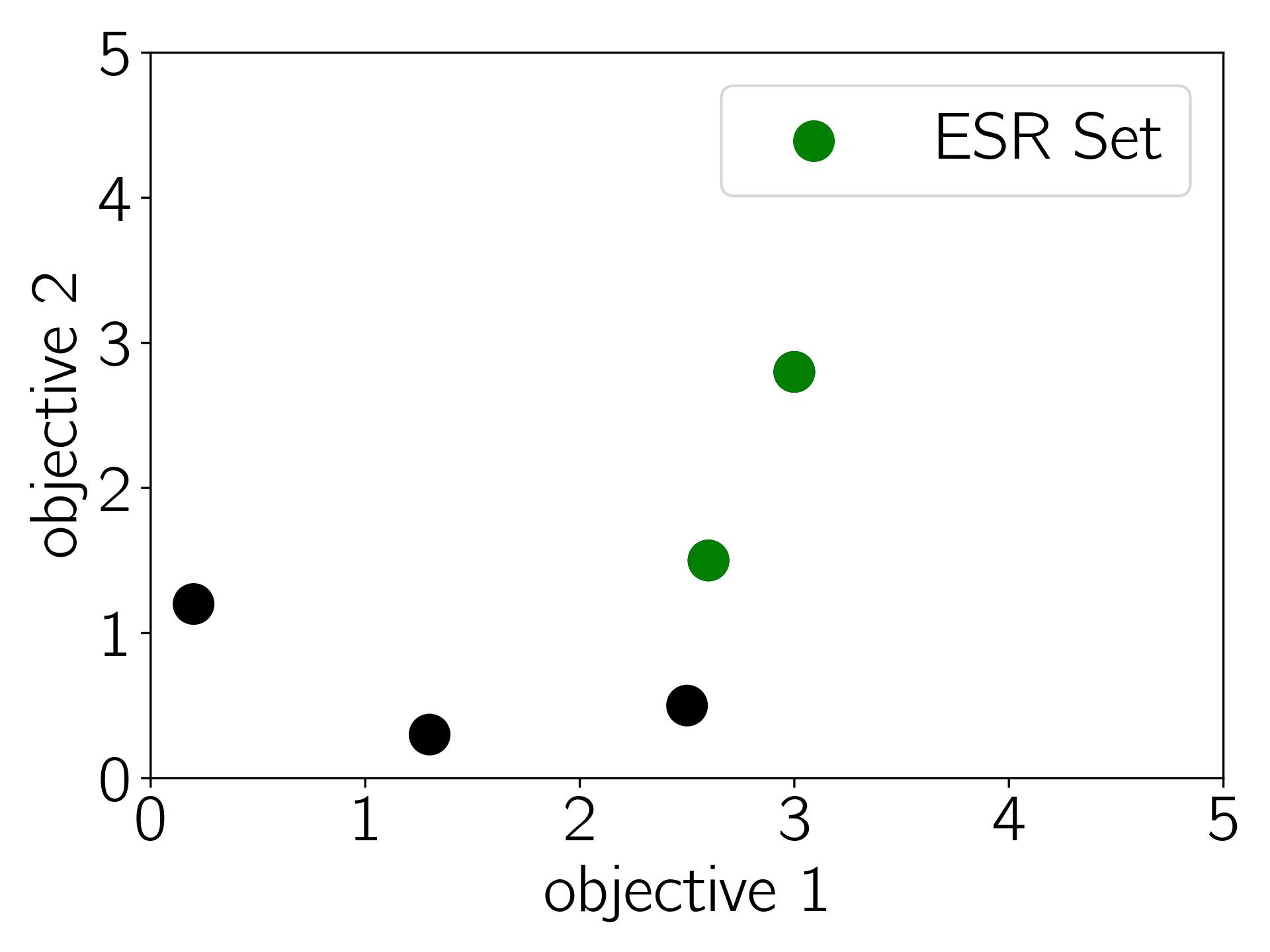}
    \end{subfigure}
    \caption{The policies on the Pareto front (left) are different from the expectations of the policies in the ESR set (right). In this case, a policy that is in the ESR set is not on the Pareto front. This figure illustrate why SER methods cannot be used to learn the ESR set.} 
    \label{fig:esr_pareto}
\end{figure}

Figure \ref{fig:esr_pareto} highlights why the choice of optimality criteria must be take into consideration for multi-objective decision making when the utility function of the user is unknown. A number of SER methods use Pareto dominance to determine a partial ordering over policies. The Pareto dominant policies, or Pareto front, are then returned to the user. To determine the Pareto front \cite{pareto1896dominance} the expectations of each arm in the MOMAB setting are calculated and the Pareto dominant policies are determined. In Figure \ref{fig:esr_pareto} the policies on the Pareto front (left) have been highlighted in red, all other policies are Pareto dominated. In the MOMAB environment outlined in Table \ref{table:momab_distributions}, the Pareto front consists of a single policy. Figure \ref{fig:esr_pareto} (right) displays the expected values of the policies in the ESR set, highlighted in green. By comparing both plots in Figure \ref{fig:esr_pareto}, it is clear that  the ESR set contains an extra policy. Therefore, in some settings, certain policies that are optimal under the ESR criterion are dominated under the SER criterion. Figure \ref{fig:esr_pareto} highlights the importance of selecting the correct optimality criterion when learning. If SER methods are used to compute a set of optimal policies in scenarios where the ESR criterion should be used, it is possible a sub-optimal policy may be selected by the user at decision time. This may have adverse affects when applying multi-policy multi-objective methods in real-world decision making settings.

\subsection{Vaccine Recommender System}
To illustrate a potential real-world use case for the ESR criterion and ESR dominance, we define a new multi-objective multi-armed bandit environment known as the Vaccine Recommender System. For example, in a medical setting a doctor may only have one opportunity to select a treatment for a patient. In this case it is necessary to optimise under the ESR criterion. Consider the following scenario: a patient is travelling to another country where it is required to be vaccinated for a specific disease to gain entry to the country. There are five available vaccines, however, each vaccine will have varying side effects (safety rating) and effectiveness. This problem has two objectives: safety and effectiveness. Both objectives are ranked from 0 to 5, with 0 being the worst rating and 5 being the best rating. None of the available vaccines are 100$\%$ effective at treating the disease. When taking each vaccine there is a chance of different outcomes occurring, for example, there is a chance of having severe side effects (low safety rating) and a chance of the vaccine providing the required immunity to the disease (high effectiveness rating). Table \ref{table:vaccine_distributions} outlines each vaccine and the probability of each outcome occurring after taking the vaccine. Table \ref{table:vaccine_distributions} is unknown to the agent, and the agent aims to learn each distribution per vaccine and prune the ESR dominated vaccines from consideration.
\begin{table}

    \centering
    \begin{tabular}{| c | c |}
    \multicolumn{2}{c}{Vaccine 1 ($V_{1}$)} \\
    \hline
         P($V_{1}$= $\textbf{R}$) & \textbf{R} \\
    \hline
         0.05 & (2, 0) \\
    \hline 
          0.05 & (2, 1) \\
    \hline 
          0.1 & (3, 2) \\
     \hline 
          0.8 & (4, 2) \\
    \hline
    \end{tabular}
    \begin{tabular}{| c | c |}
    \multicolumn{2}{c}{Vaccine 2 ($V_{2}$)} \\
    \hline
         P($V_{2}$= $\textbf{R}$) & \textbf{R} \\
    \hline
         0.1 & (0, 0) \\
    \hline 
          0.1 & (1, 1) \\
    \hline 
          0.5 & (2, 0) \\
     \hline 
          0.3 & (2, 1) \\
    \hline
    \end{tabular}
    \begin{tabular}{| c | c |}
    \multicolumn{2}{c}{Vaccine 3 ($V_{3}$)} \\
    \hline
         P($V_{3}$= $\textbf{R}$) & \textbf{R} \\
    \hline
         0.1 & (1, 0) \\
    \hline 
          0.1 & (1, 3) \\
    \hline 
          0.2 & (3, 4) \\
     \hline 
          0.6 & (5, 4) \\
    \hline
    \end{tabular}
    \begin{tabular}{| c | c |}
    \multicolumn{2}{c}{Vaccine 4 ($V_{4}$)} \\
    \hline
         P($V_{4}$= $\textbf{R}$) & \textbf{R} \\
    \hline
         0.1 & (1, 0) \\
    \hline 
          0.4 & (2, 1) \\
    \hline 
          0.4 & (3, 1) \\
     \hline 
          0.1 & (3, 2) \\
    \hline
    \end{tabular}
    \begin{tabular}{| c | c |}
    \multicolumn{2}{c}{Vaccine 5 ($V_{5}$)} \\
    \hline
         P($V_{5}$= $\textbf{R}$) & \textbf{R} \\
    \hline 
          0.8 & (0, 0) \\
    \hline
         0.05 & (1, 1) \\
    \hline 
          0.05 & (1, 2) \\
    \hline 
          0.1 & (4, 0) \\
    \hline
    \end{tabular}
    \caption{A group of available vaccines that have varying outcomes. Some vaccines have a higher chance of side effects (low safety rating), while others are more effective at providing immunity. The objectives are ordered as follows: $\textbf{R}$ = (safety, effectiveness).}
    \label{table:vaccine_distributions}
\end{table}

Given the utility function of the user is unknown, the MOTDRL algorithm is used to learn the underlying return distributions for each vaccine in Table \ref{table:vaccine_distributions} and determine the ESR set. Once MOTDRL has finished learning a set of optimal polices, in this case the ESR set, is returned to the user. When the user's utility function becomes known, a vaccine that maximises the user's utility function can be selected from the ESR set by the user.

The ESR set for the Vaccine Recommender System (VRS) environment is known a priori. The return distributions for $V_{1}$ and $V_{3}$ are ESR dominant and therefore $V_{1}$ and $V_{3}$ are the only distributions in the ESR set. The VRS environment has five arms where each arm corresponds to a vaccine in Table \ref{table:vaccine_distributions}. To evaluate MOTDRL in a VRS environment, we set $\textbf{R}_{min} = 0$, $\textbf{R}_{max} = 10$, $D$ = 2, $\beta$ = 5 and $|E^{*}|$ = 2. All experiments in this setting are averaged over 10 runs and each experiment lasts $200,000$ episodes. To compute the coverage ratio, we set $\epsilon = 0.01$. 
\begin{figure}[h]
    \centering
    \includegraphics[height = 7cm, width=9cm]{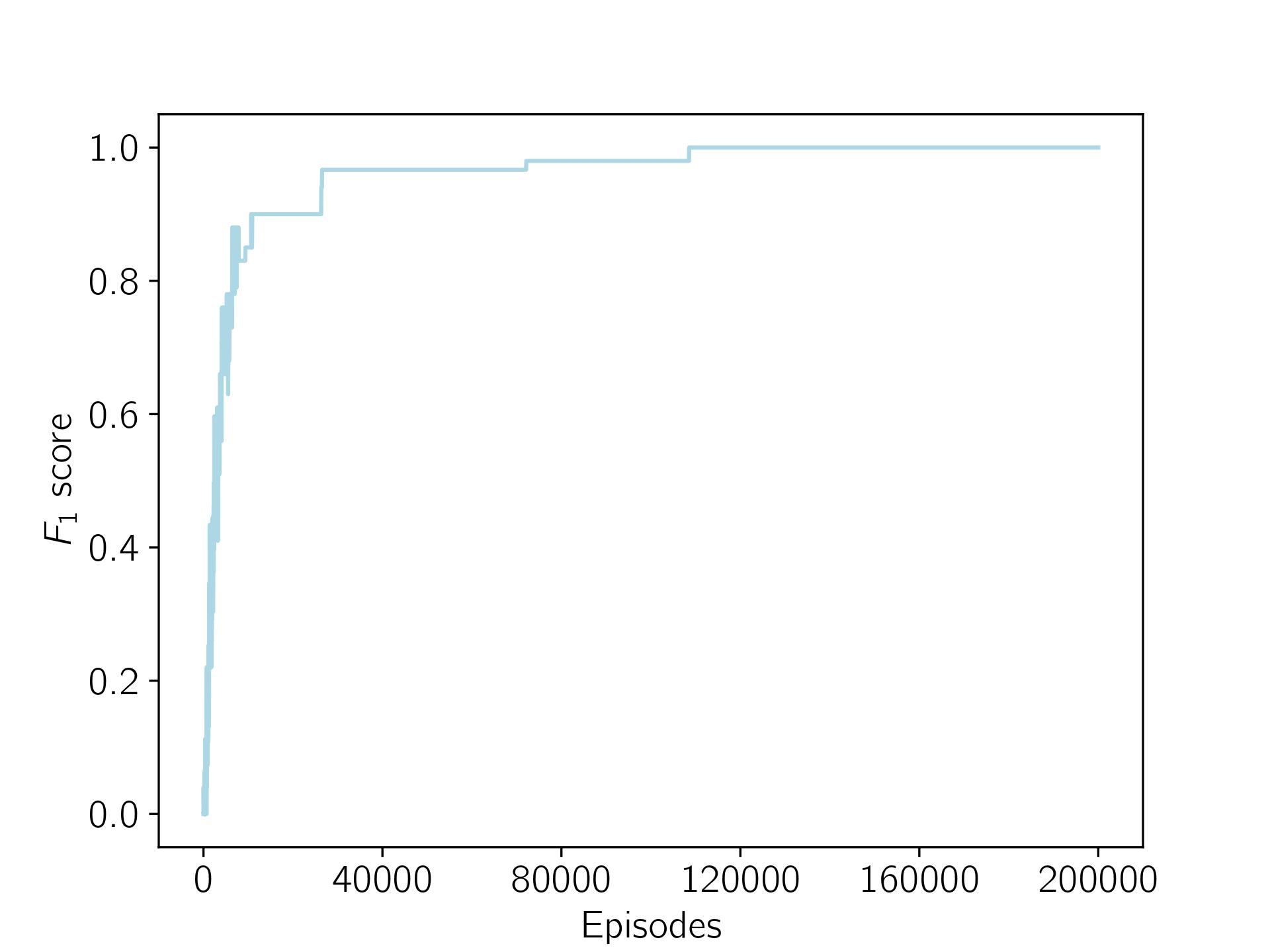}
    \caption{Results from the VRS environment. MOTDRL is able to learn the full ESR set as it converges the optimal $F_1$ score of 1.} 
    \label{fig:vaccine_f1}
\end{figure}

After sufficient sampling, MOTDRL is able to learn the underlying return distributions for each arm  in the VRS environment. Given return distributions can be used to give a partial ordering over policies, MOTDRL can use the return distributions for each arm to compute the ESR set in the VRS environment. In Figure \ref{fig:vaccine_f1}, we plot the coverage ratio as the $F_1$ score. MOTDRL converges to the optimal $F_1$ score after $120,000$ episodes. Given MOTDRL converges to the optimal $F_1$ score it is clear MOTDRL is able to learn the ESR set. 

\begin{figure}[h]
    \centering
    \begin{subfigure}{.5\textwidth}
         \includegraphics[width=6.25cm]{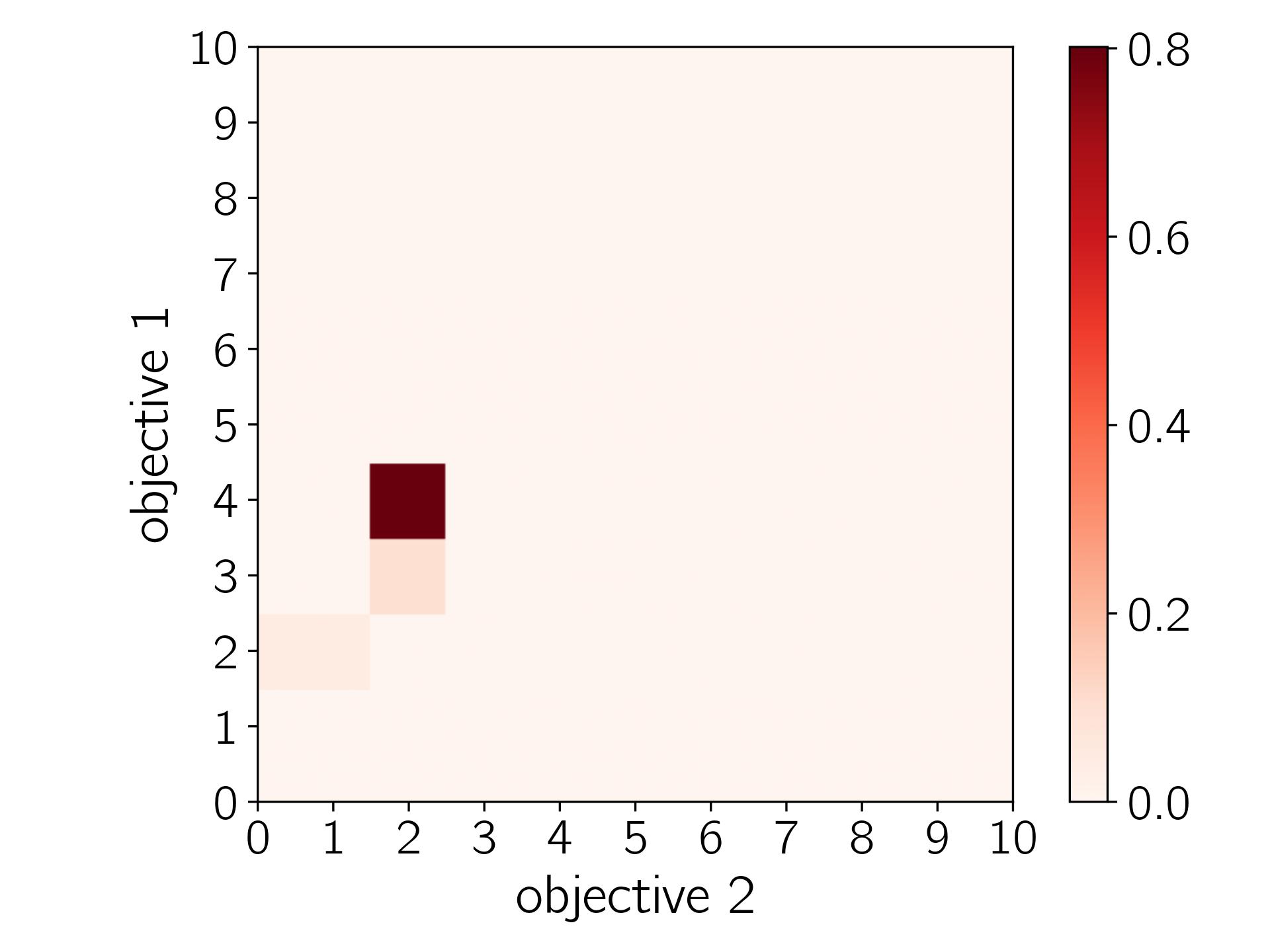}
    \end{subfigure}%
    \begin{subfigure}{.5\textwidth}
    \includegraphics[width=6.25cm]{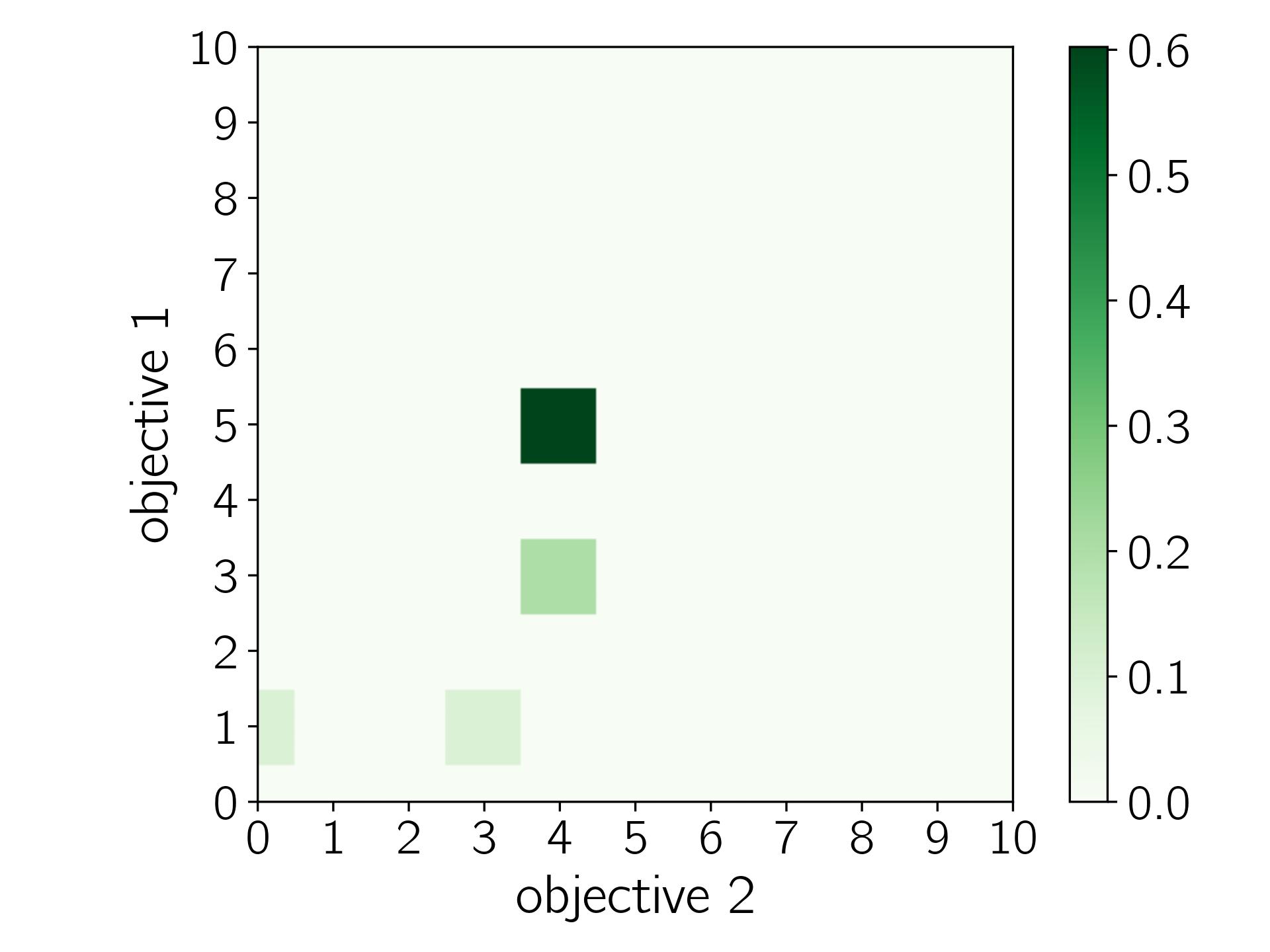}
    \end{subfigure}
    \caption{Heatmaps for each policy in the ESR set learned by MOTDRL. The left heatmap describes the distribution for $V_{1}$ learned by MOTDRL and the right heatmap describes the distribution for $V_{3}$ learned by MOTDRL.} 
    \label{fig:vaccine_heatmaps}
\end{figure}

In practice, once learning has completed, MOTDRL returns the learned ESR set for the VRS environment to the user. The learned ESR set contains two vaccines; $V_{1}$ and $V_{3}$. Both vaccines in the ESR set are ESR dominant. Moreover, a user with a monotonically increasing utility function will prefer either $V_{1}$ or $V_{3}$ over all other vaccines in the VRS environment.

Similarly to Section \ref{sec:experiments_momab}, we utilise Figure \ref{fig:vaccine_heatmaps} and Figure \ref{fig:vaccine_cdfs} to give the reader some intuition about ESR dominance. Figure \ref{fig:vaccine_heatmaps} presents heatmaps to represent the policies in the ESR set learned by MOTDRL. Each heatmap represents a return distribution learned by MOTDRL and shows the return vectors and the corresponding probabilities. Each heatmap in Figure \ref{fig:vaccine_heatmaps} corresponds to the probabilities highlighted for $V_{1}$ (left) and $V_{3}$ (right) in Table \ref{table:vaccine_distributions}. Figure \ref{fig:vaccine_cdfs} displays the policies in the ESR set learned by MOTDRL and their corresponding CDFs. Each CDF in Figure \ref{fig:vaccine_cdfs} corresponds to the CDFs of the underlying return distributions of $V_{1}$ and $V_{3}$ in Table \ref{table:vaccine_distributions}. 

\begin{figure}
    \centering
    \begin{subfigure}{.5\textwidth}
         \includegraphics[width=6cm]{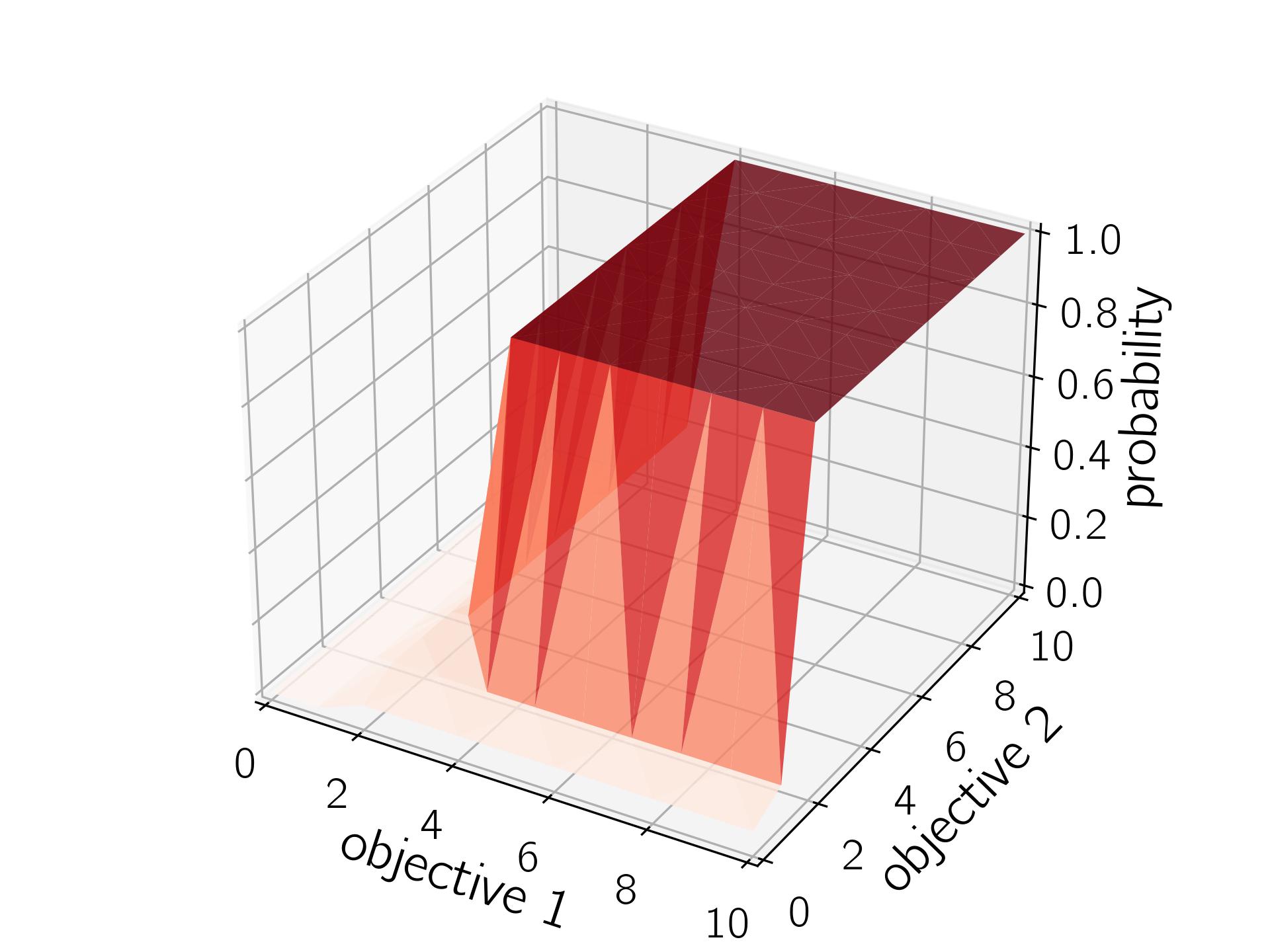}
    \end{subfigure}%
    \begin{subfigure}{.5\textwidth}
    \includegraphics[width=6cm]{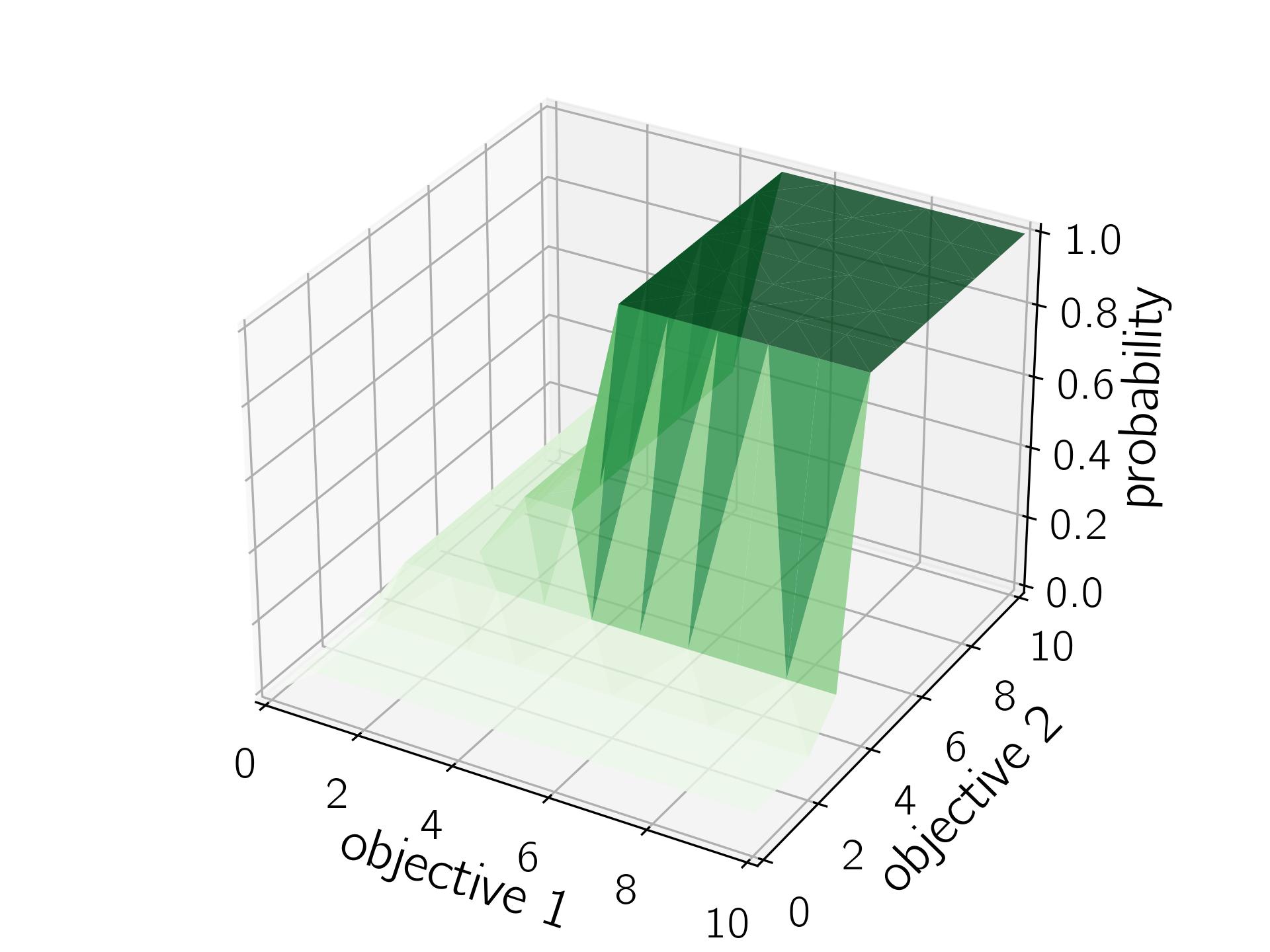}
    \end{subfigure}
    \caption{CDFs for each policy in the ESR set learned by MOTDRL in the VRS environment. The left figure describes the CDF for $V_{1}$ learned by MOTDRL and the right figure describes the CDF for $V_{3}$ learned by MOTDRL.} 
    \label{fig:vaccine_cdfs}
\end{figure}


\section{Related Work}
\label{sec:related}
The various orders of stochastic dominance have been used extensively as a method to determine the optimal decision when making decisions under uncertainty in economics \cite{choi1988economics}, finance \cite{ali1975finance,bawa1978finance}, game theory \cite{fishburn1978ngametheory}, and various other real-world scenarios \cite{bawa1982realworld}. However, stochastic dominance has largely been overlooked in systems that learn. Cook and Jarret \cite{cook2018mo_sd} use various orders of stochastic dominance and Pareto dominance with genetic algorithms to compute optimal solution sets for an aerospace design problem with multiple objectives when constrained by a computational budget. Martin et al. \cite{martin2020stochastically} use second-order stochastic dominance (SSD) with a single-objective distributional RL algorithm \cite{bellemare2017distributional}. Martin et al. \cite{martin2020stochastically} use SSD to determine the optimal action to take at decision time, and this approach is shown to learn good policies during experimentation.

To learn the ESR set in sequential decision making processes, like MOMDPs, new distributional MORL methods must be formulated. Distributional Monte Carlo tree search (DMCTS) is a state-of-the-art ESR method and uses a bootstrap Thompson sampling method to approximate a posterior distribution over the returns \cite{hayes2021dmcts}. However, this method is a single policy method and relies on the utility function of the user to be known at the time of learning or planning. DMCTS would invalidate the ESR criterion in the unknown utility function scenario and would therefore be unable to learn the ESR set. Distributional methods like the C51 algorithm, proposed by Bellemare et al. \cite{bellemare2017distributional}, could potentially be used to learn the underlying distribution of a random vector. However C51 is a single-objective method and defining a multi-objective version of C51 to learn the ESR set could pose significant challenges. Replacing the distribution over returns used by C51 with a multi-variate distribution could cause computation to increase with the number of objectives. In this case, dedicated multi-objective distributional methods must be formulated so that it is possible to efficiently learn the ESR set for the ESR criterion. We highlight this as a new challenge that must be addressed by the MORL community. 

\section{Conclusion \& Future Work} 
\label{sec:conclusion}
MORL has been highlighted as one of several key challenges that needs to be addressed in order for RL to be commonly deployed in real-world systems \cite{dulac-arnold2021}. In order to apply RL to the real world the MORL community must consider the ESR criterion. However, the ESR criterion has largely been ignored by the MORL community, with the exception of the works of Roijers et al. \cite{roijers2013survey,roijers2018multi}, Hayes et al. \cite{hayes2021dmcts,hayes2021dmcts_long} and Vamplew et al. \cite{vamplew2021}. The works of Hayes et al. \cite{hayes2021dmcts_long,hayes2021dmcts} and Roijers et al. \cite{roijers2018multi} present single-policy algorithms that are suitable to learn policies under the ESR criterion, however, prior to this work, a formal definition of the necessary requirements to compute policies under the ESR criterion had not previously been defined. In Section \ref{sec:esr}, we outline, through examples and definitions, the necessary requirements to optimise under the ESR criterion.
The formal definitions outlined in Section \ref{sec:esr} ensure that an optimal policy can be learned when the utility function of the user is known under the ESR criterion. However, in the real world, a user's preferences over objectives (or utility function) may be unknown at the time of learning \cite{roijers2013survey}.

Prior to this paper, a suitable solution set for the unknown utility function scenario under the ESR criterion had not been defined. This long-standing research gap has restricted the applicability of MORL in real-world scenarios under the ESR criterion. In Section \ref{sec:stochastic_dominance_esr} and Section \ref{sec:esr_solution_sets} we define the necessary solution sets required for multi-policy algorithms to learn a set of optimal policies under the ESR criterion when the utility function of a user is unknown. In Section \ref{sec:modrl} we present a novel multi-policy algorithm, known as multi-objective tabular distributional reinforcement learning (MOTDRL), that can learn the ESR set in a MOMAB setting when the utility function of a user is unknown at the time of learning. In Section \ref{sec:experiments} we evaluate MOTDRL in two MOMAB settings and show that MOTDRL can learn the ESR set in MOMAB settings.
This work aims to answer some of the existing research questions regarding the ESR criterion. Moreover, we aim to highlight the importance of the ESR criterion when applying MORL to real-world scenarios. In order to successfully apply MORL to the real world, we must implement new single-policy and multi-policy algorithms that can learn solutions for non-linear utility functions in various scenarios.

A promising  direction for future work would be to extend the work of Hayes et al. \cite{hayes2021dmcts} and the work of Wang and Sebag \cite{wang2012multi}. It may be possible to build on the aforementioned works to implement a multi-objective distributional Monte Carlo tree search algorithm that can learn a set of optimal policies under the ESR criterion. It is important to note that Hayes et al. \cite{hayes2021dmcts_long,hayes2021dmcts} use bootstrap Thompson sampling to approximate a posterior distribution. This method cannot learn the ESR set when utility function of a user is unknown, therefore a different distributional method must be used to learn the ESR set. Although the distributional method used by Hayes et al. \cite{hayes2021dmcts} cannot be used to learn the ESR set, this work is still a useful starting point.

Given distributional MORL methods are a new area of research, not much is known about the computational requirements of maintaining a return distribution. Therefore, it is important that a comprehensive computational analysis of distributional MORL methods is undertaken to fully understand the implications of maintaining a return distribution. In a future publication we plan to perform a computational analysis for distributional MORL methods in both bandit and sequential decision making settings.

A lack of well defined benchmarks is a significant challenge associated with implementing any new single-policy or multi-policy algorithms under the ESR criterion. Currently, very few ESR benchmark environments exist (e.g. Fishwood \cite{roijers2018multi}). In order to accurately evaluate single-policy and multi-policy ESR algorithms, a suite of benchmark problem domains need to be designed. Under the SER criterion, such benchmarks already exist, e.g. Deep Sea Treasure \cite{vamplew2011evaluation_methods}.
It is also important to highlight the need to establish new metrics to evaluate multi-policy algorithms under the ESR criterion. As previously mentioned, all metrics used to evaluate multi-objective algorithms are designed for the SER criterion. In order to accurately evaluate multi-policy algorithms under the ESR criterion, new metrics must be determined. We note that extending the work of Zintgraf et al. \cite{zintgraf2015quality_ser} for the ESR criterion would be a promising starting point. 

\section{Supplementary Material}
\label{sec:supplemental_material}

\begin{lemma}
(Beppo Levi's lemma \cite{Schappacher1995BeppoLA}) Consider a point-wise non-decreasing sequence of positive functions $f_n\ :\ X \rightarrow [0, +\infty]$, i.e., for every $k \ge 1$ and every $x \in X$.
\begin{equation*}
0 \le f_n(x) \le f_{n+1}(x) \le +\infty
\end{equation*}
Set the point-wise limit of the sequence $\{f_i\}$ to be $f$. That is, for every $x \in X$,
\begin{equation*}
\lim_{n \rightarrow +\infty} f_n(x) = f(x)
\end{equation*}
Then $f$ is measurable and:
\begin{equation*}
\lim_{n \rightarrow +\infty} \int f_n(x) dx = \int \lim_{n \rightarrow +\infty}  f_n(x) dx 
\end{equation*}
\label{lemma:beppo_levi}
\end{lemma}

\begin{lemma} (Monotone convergence)
Let $u$ be a non-negative monotonically increasing utility function in $x$ and $y$, and $F$ the CDF of a random variables $X$ and $Y$. Then,
\begin{equation*}
\int \lim_{y \rightarrow +\infty} u(x, y) F(x, y) dx = \lim_{y \rightarrow +\infty} \int u(x, y) F(x, y) dx.
\end{equation*}
\label{lemma:monotone_convergence}
\end{lemma}
\begin{proof}
Let $g_n(x) = u(x, n) F(x, n)$. As $u$ and $F$ are positive monotonically increasing functions in $n$, the function $g_n$ is also positive and monotonically increasing, i.e.,
\begin{equation*}
0 \le g_n(x) \le g_{n+1}(x) \le +\infty
\end{equation*}
According to Beppo Levi's lemma (see Lemma~\ref{lemma:beppo_levi}), the limit of the integral of $g_n(x)$ in $x$ is the integral of its limit, i.e.,
\begin{equation*}
\lim_{n \rightarrow +\infty} \int g_n(x) dx = \int \lim_{n \rightarrow +\infty} g_n(x) dx.
\end{equation*}
\end{proof}


\section*{Acknowledgements}
Conor F.\ Hayes is funded by the National University of Ireland Galway Hardiman Scholarship. This research was supported by funding from the Flemish Government under the ``Onderzoeksprogramma Artificiële Intelligentie (AI) Vlaanderen'' program. 

\section*{Conflict of interest}
The authors declare that they have no conflict of interest.

\bibliographystyle{unsrt}  
\bibliography{references}

\end{document}